\documentclass{article}

\pdfoutput=1

\usepackage{microtype}
\usepackage{graphicx}
\usepackage{subfigure}
\usepackage{booktabs} 

\usepackage{float}

\usepackage[dvipsnames]{xcolor}
\definecolor{pearThree}{HTML}{E74C3C}
\definecolor{pearcomp}{HTML}{B97E29}
\definecolor{pearDark}{HTML}{2980B9}
\definecolor{pearDarker}{HTML}{1D2DEC}
\usepackage[hidelinks]{hyperref}
\hypersetup{
	colorlinks,
	citecolor=pearDark,
	linkcolor=pearThree,
	breaklinks=true,
	urlcolor=pearDarker}

\usepackage{footmisc}

\PassOptionsToPackage{numbers}{natbib}
\usepackage[final]{neurips_2020}

\usepackage{url}            
\usepackage{booktabs}       
\usepackage{amsfonts}       
\usepackage{nicefrac}       
\usepackage{microtype}      
\usepackage{amsmath,amsthm}
\usepackage{mathtools}
\usepackage[ruled,vlined,linesnumbered]{algorithm2e} 
\usepackage{enumitem}
\usepackage{dirtytalk}
\usepackage{tikz}
\usetikzlibrary{positioning,angles,quotes}
\usetikzlibrary{arrows.meta,calc}
\usepackage{dsfont}
\usepackage{amssymb}
\usepackage{upgreek}
\usepackage{multirow}
\usepackage{soul}
\usepackage{framed}
\usepackage{pifont}
\usepackage{xspace}
\usepackage{float}
\usepackage{empheq}

\usepackage{scalerel}

\usetikzlibrary{matrix, positioning, arrows.meta}
\newlength{\myheight}
\tikzset{labels/.style={font=\sffamily\scriptsize},
    circuit/.style={draw,minimum width=2cm,minimum height=\myheight,very thick,inner sep=1mm,outer sep=0pt,cap=round,font=\sffamily\bfseries},
    triangle 45/.tip={Triangle[angle=45:8pt]}
}

\usepackage{multicol}
\usepackage{textgreek}
\makeatletter
\newcommand{\greekalpha}[1]{\c@greekalpha{#1}}
\newcommand{\c@greekalpha}[1]{%
  {%
    \ifcase\number\value{#1} %
    \or
    \textalpha
    \or
    \textbeta
    \or
    \textgamma
    \fi
  }%
}

\AddEnumerateCounter*{\greekalpha}{\c@greekalpha}{5}

\newcommand{\CommaBin}{\mathbin{\raisebox{0.5ex}{,}}}
\let\originalleft\left
\let\originalright\right
\renewcommand{\left}{\mathopen{}\mathclose\bgroup\originalleft}
\renewcommand{\right}{\aftergroup\egroup\originalright}
\newcommand{\wt}[1]{\widetilde{#1}}

\newcommand{\wh}[1]{\widehat{#1}}

\newcommand{\StopOne}{{\small\textsc{STOP1}}\xspace}
\newcommand{\StopTwo}{{\small\textsc{STOP2}}\xspace}

\newcommand{\Lpluseps}{\scaleto{L + \epsilon}{5.5pt}}

\newcommand{\DC}{{\small\textsc{AX}}\xspace}
\newcommand{\DCtitle}{\textsc{AX}\xspace}

\newcommand{\DCL}{{\small\textsc{AX}\textsubscript{L}}\xspace}
\newcommand{\DCstar}{{\small\textsc{AX}$^{\star}$}\xspace}
\newcommand{\DCprime}{{\small\textsc{AX}$'$}\xspace}

\newcommand{\cS}{\mathcal{S}}
\newcommand{\cA}{\mathcal{A}}
\newcommand{\cK}{\mathcal{K}}
\newcommand{\SA}{\mathcal{S} \times \mathcal{A}}

\newcommand{\ALGO}{\textup{\texttt{DisCo}}\xspace}
\newcommand{\ALGOtitle}{\large{\textup{\texttt{DisCo}}}\xspace}

\newcommand{\UcbExplore}{\textup{\texttt{UcbExplore}}\xspace}
\newcommand{\UcbExplorelarge}{\large{\textup{\texttt{UcbExplore}}}\xspace}

\newcommand{\Ucb}{{\small{\textsc{Ucb}}}\xspace}

\newcommand{\MNM}{{\small\textsc{MNM}}\xspace}

\newcommand{\OptMod}{{\small\textsc{OptMod}}\xspace}

\newcommand{\VI}{{\small\textsc{VI}}\xspace}
\newcommand{\VISSP}{{\small$\textsc{VI}_{\textsc{SSP}}$}\xspace}

\newcommand{\OVISSP}{{\small$\textsc{OVI}_{\textsc{SSP}}$}\xspace}
\newcommand{\OVISSPmath}{{\small\textsc{OVI}_{\textsc{SSP}}}\xspace}

\newcommand{\RESET}{{\small\textsc{RESET}}\xspace}

\DeclareMathOperator*{\argmin}{arg\,min}

\newcommand\myeqa{\mathrel{\stackrel{\makebox[0pt]{\mbox{\normalfont\tiny (a)}}}{=}}}

\newcommand\myineeqa{\mathrel{\stackrel{\makebox[0pt]{\mbox{\normalfont\tiny (a)}}}{\leq}}}
\newcommand\myineeqb{\mathrel{\stackrel{\makebox[0pt]{\mbox{\normalfont\tiny (b)}}}{\leq}}}
\newcommand\myineeqc{\mathrel{\stackrel{\makebox[0pt]{\mbox{\normalfont\tiny (c)}}}{\leq}}}
\newcommand\myineeqd{\mathrel{\stackrel{\makebox[0pt]{\mbox{\normalfont\tiny (d)}}}{\leq}}}

\makeatletter
\newcommand\footnoteref[1]{\protected@xdef\@thefnmark{\ref{#1}}\@footnotemark}
\makeatother

\DeclarePairedDelimiter\abs{\lvert}{\rvert}%
\DeclarePairedDelimiter\norm{\lVert}{\rVert}%

\newtheorem{theorem}{Theorem}
\newtheorem{lemma}{Lemma}

\newtheorem{corollary}{Corollary}
\newtheorem{definition}{Definition}
\newtheorem{assumption}{Assumption}
\theoremstyle{definition}

\renewcommand{\epsilon}{\varepsilon}
\renewcommand{\hat}{\widehat}

\usepackage[colorinlistoftodos, textwidth=30mm]{todonotes}

\newcommand{\todojout}[1]{\todo[color=orange!30!yellow!10]{\scriptsize JT: #1}}

\usepackage{minitoc}
\renewcommand \thepart{}
\renewcommand \partname{}

\newcommand{\jt}[1]{\textcolor{black}{#1}}

\title{Improved Sample Complexity for Incremental Autonomous Exploration in MDPs}

\author{%
  Jean Tarbouriech \hspace{-4em}\\
  Facebook AI Research Paris \& Inria Lille\hspace{-5em}\\
  \hspace{-4em}\texttt{jean.tarbouriech@gmail.com\hspace{-5em}}
  \And
  \hspace{-2em}Matteo Pirotta \\
  \hspace{-2em}Facebook AI Research Paris\\
  \hspace{-2em}\texttt{pirotta@fb.com}
  \AND
  \hspace{3em}Michal Valko \\
  \hspace{3em}DeepMind Paris\\
  \hspace{3em}\texttt{valkom@deepmind.com} \\
  \And
  \hspace{2em}Alessandro Lazaric \\
  \hspace{2em}Facebook AI Research Paris\\
  \hspace{2em}\texttt{lazaric@fb.com}
}

\begin{document}

\maketitle

\vspace{-0.1in}
	\doparttoc 
	\faketableofcontents 

\begin{abstract}%
    We investigate the exploration of an unknown environment when no reward function is provided. Building on the incremental exploration setting introduced by Lim and Auer~\cite{lim2012autonomous}, we define the objective of learning the set of $\epsilon$-optimal goal-conditioned policies attaining all states that are incrementally reachable within $L$ steps (in expectation) from a reference state $s_0$. In this paper, we introduce a novel model-based approach that interleaves discovering new states from $s_0$ and improving the accuracy of a model estimate that is used to compute goal-conditioned policies to reach newly discovered states. The resulting algorithm, \ALGO, achieves a sample complexity scaling as $\widetilde{O}(L^5 S_{\Lpluseps} \Gamma_{\Lpluseps} A \,\epsilon^{-2})$, where~$A$ is the number of actions, $S_{\Lpluseps}$ is the number of states that are incrementally reachable from $s_0$ in $L+\epsilon$ steps, and $\Gamma_{\Lpluseps}$ is the branching factor of the dynamics over such states. This improves over the algorithm proposed in \citep{lim2012autonomous} in both~$\epsilon$ and $L$ at the cost of an extra $\Gamma_{\Lpluseps}$ factor, which is small in most environments of interest. Furthermore, \ALGO is the first algorithm that can return an $\epsilon/c_{\min}$-optimal policy for any cost-sensitive shortest-path problem defined on the $L$-reachable states with minimum cost $c_{\min}$. Finally, we report preliminary empirical results confirming our theoretical findings.
\end{abstract}



\vspace{-0.12in}
\section{Introduction}
\vspace{-0.02in}

In cases where the reward signal is not informative enough --- e.g., too sparse, time-varying or even absent --- a reinforcement learning (RL) agent needs to explore the environment driven by objectives other than reward maximization, see~\citep[e.g.,][]{schmidhuber1991possibility, chentanez2005intrinsically, oudeyer2009intrinsic, singh2010intrinsically, baranes2010intrinsically}. This can be performed by designing intrinsic rewards to drive the learning process, for instance via state visitation counts~\citep{bellemare2016unifying, tang2017exploration}, novelty or prediction errors~\citep{houthooft2016variational, pathak2017curiosity,azar2019world}. Other recent methods perform information-theoretic skill discovery to learn a set of diverse and task-agnostic behaviors~\citep{eysenbach2018diversity, sharma2019dynamics, campos2020explore}. Alternatively, goal-conditioned policies learned by carefully designing the sequence of goals during the learning process are often used to solve sparse reward problems~\cite{ecoffet2020first} and a variety of goal-reaching tasks~\citep{florensa2018automatic, colas2019curious, warde2018unsupervised, pong2019skew}.

\jt{While the approaches reviewed above effectively leverage deep RL techniques and are able to achieve impressive results in complex domains (e.g., Montezuma's Revenge~\cite{ecoffet2020first} or real-world robotic manipulation tasks~\cite{pong2019skew}), they often lack substantial theoretical understanding and guarantees. Recently, some \textit{unsupervised RL} objectives were analyzed rigorously. Some of them quantify how well the agent visits the states under a sought-after frequency, e.g., to induce a maximally entropic state distribution~\citep{hazan2019provably, tarbouriech2019active, cheung2019exploration, tarbouriech2020active}. While such strategies provably mimic their desired behavior via a Frank-Wolfe algorithmic scheme, they may not learn how to effectively reach any state of the environment and thus may not be sufficient to efficiently solve downstream tasks. Another relevant take is the reward-free RL paradigm of~\citep{jin2020reward}: following its exploration phase, the agent is able to compute a near-optimal policy for any reward function at test time. While this framework yields strong end-to-end guarantees, it is limited to the finite-horizon setting and the agent is thus unable to tackle tasks beyond finite-horizon, e.g., goal-conditioned tasks.}

In this paper, we build on and refine the setting of incremental exploration of~\citep{lim2012autonomous}: the agent starts at an initial state $s_0$ in an unknown, possibly large environment, and it is provided with a \textsc{reset} action to restart at $s_0$. \jt{At a high level, in this setting the agent should explore the environment and stop when it has identified the \textit{tasks} within its \textit{reach} and learned to \textit{master} each of them sufficiently well.} More specifically, the objective of the agent is to learn a goal-conditioned policy for \textit{any} state that can be reached from $s_0$ within $L$ steps in expectation; such a state is said to be $L$-controllable. Lim and Auer~\citep{lim2012autonomous} address 
this setting with the \UcbExplore method for which they bound the number of exploration steps that are required to identify in an incremental way all $L$-controllable states (i.e., the algorithm needs to define a suitable stopping condition) and to return a set of policies that are able to reach each of them in \textit{at most} $L+\epsilon$ steps. A key aspect of \UcbExplore is to first focus on simple states (i.e., states that can be reached within a few steps), learn policies to efficiently reach them, and leverage them to identify and tackle states that are increasingly more difficult to reach. This approach aims to avoid wasting exploration in the attempt of reaching states that are further than $L$ steps from $s_0$ or that are too difficult to reach given the limited knowledge available at earlier stages of the exploration process. Our  main contributions are: 
\begin{itemize}[leftmargin=.2in,topsep=-4pt,itemsep=0pt,partopsep=0pt, parsep=0pt]
	\setlength\itemsep{0em}
	\item We strengthen the objective of incremental exploration and require the agent to learn $\epsilon$-optimal goal-conditioned policies for any $L$-controllable state. Formally, let $V^\star(s)$ be the length of the shortest path from $s_0$ to $s$, then the agent needs to learn a policy to navigate from $s_0$ to $s$ in at most $V^\star(s)+\epsilon$ steps, while in~\citep{lim2012autonomous} any policy reaching $s$ in \textit{at most} $L+\epsilon$ steps is acceptable.
	\item We design \ALGO, a novel algorithm for incremental exploration. \ALGO relies on an estimate of the transition model to compute goal-conditioned policies to the states observed so far and then use those policies to improve the accuracy of the model and incrementally discover new states.
	\item We derive a sample complexity bound for \ALGO scaling as\footnote{We say that $f(\epsilon) = \wt O(\epsilon^\alpha)$ if there are constants $a$, $b$, such that $f(\epsilon) \leq a \cdot \epsilon^\alpha \log^b\big(\epsilon\big)$.} $\widetilde{O}(L^5 S_{\Lpluseps} \Gamma_{\Lpluseps} A \, \epsilon^{-2})$, where $A$ is the number of actions, $S_{\Lpluseps}$ is the number of states that are \textit{incrementally} controllable from $s_0$ in $L+\epsilon$ steps, and $\Gamma_{\Lpluseps}$ is the branching factor of the dynamics over such incrementally controllable states. Not only is this sample complexity obtained for a more challenging objective than \UcbExplore, but it also improves in both $\epsilon$ and $L$ at the cost of an extra $\Gamma_{L+\epsilon}$ factor, which is small in most environments of interest.
	\item Leveraging the model-based nature of \ALGO, we can also readily compute an $\epsilon/c_{\min}$-optimal policy for \textit{any} cost-sensitive shortest-path problem defined on the $L$-controllable states with minimum cost $c_{\min}$. This result serves as a goal-conditioned counterpart to the reward-free exploration framework defined by Jin et al.\,\cite{jin2020reward} for the finite-horizon setting.    
\end{itemize}

\vspace{-0.015in}
\section{Incremental Exploration to Discover and Control}
\label{section_span_problem}
\vspace{-0.015in}
In this section we expand 
\cite{lim2012autonomous}, with a more challenging objective for autonomous exploration.
\vspace{-0.015in}
\subsection{$L$-Controllable States}
\vspace{-0.015in}
We consider a \textit{reward-free} Markov decision process\,\citep[][Sect.\,8.3]{puterman2014markov} $M := \langle\mathcal{S}, \mathcal{A}, p, s_0 \rangle$. We assume a finite action space $\mathcal{A}$ with $A = \abs{\mathcal{A}}$ actions, and a finite, possibly large state space $\mathcal{S}$ for which an upper bound $S$ on its cardinality is known, i.e., $\abs{\mathcal{S}} \leq S$.\footnote{Lim and Auer~\cite{lim2012autonomous} originally considered a countable, possibly infinite state space; however this leads to a technical issue in the analysis of \UcbExplore (acknowledged by the authors via personal communication and explained in App.\,\ref{subsection_issue_infinite_state_space}), which disappears by considering only finite state spaces.}
Each state-action pair $(s,a) \in \mathcal{S} \times \mathcal{A}$ is characterized by an unknown transition probability distribution $p(\cdot \vert s,a)$ over next states. We denote by $\Gamma_{\mathcal{S'}} := \max_{s\in\mathcal{S'},a} \norm{ \{ p(s'|s,a) \}_{s' \in \cS'}}_0$ the largest branching factor of the dynamics over states in any subset $\mathcal{S'}\subseteq \mathcal{S}$. The environment has no extrinsic reward, and $s_0 \in \mathcal{S}$ is a designated initial state.

A deterministic stationary policy $\pi: \mathcal{S} \rightarrow \mathcal{A}$ is a mapping between states to actions and we denote by $\Pi$ the set of all possible policies. Since in environments with arbitrary dynamics the learner may get stuck in a state without being able to return to $s_0$, we introduce the following assumption.\footnote{This assumption should be contrasted with the finite-horizon setting, where each policy resets automatically after $H$ steps, or assumptions on the MDP dynamics such as ergodicity or bounded diameter, which guarantee that it is always possible to find a policy navigating between any two states.}

\begin{assumption}
	The action space contains a $\RESET$ action s.t.\ $p(s_0| s,\RESET) = 1$ for any $s \in \mathcal{S}$.
	\label{assumption_reset}
\end{assumption}

We make explicit the states where a policy $\pi$ takes action $\RESET$ in the following definition.


\begin{definition}[Policy restricted on a subset]
	For any $\mathcal{S}' \subseteq \mathcal{S}$, a policy $\pi$ is \textit{restricted on $\mathcal{S}'$} if $\pi(s) = \RESET$ for any $s \notin \mathcal{S}'$. We denote by $\Pi(\mathcal{S}')$ the set of policies restricted on $\mathcal{S}'$.
\end{definition}

We measure the performance of a policy in navigating the MDP as follows.

\begin{definition}
For any policy $\pi$ and a pair of states $(s, s') \in \mathcal{S}^2$, let $\tau_{\pi}(s \rightarrow s')$ be the (random) number of steps it takes to reach $s'$ starting from $s$ when executing policy $\pi$, i.e.,~$\tau_{\pi}(s \rightarrow s') := \inf \{ t \geq 0: s_{t+1} = s' \,\vert\, s_1 = s, \pi \}$. We also set $v_{\pi}(s \rightarrow s') := \mathbb{E}\left[\tau_{\pi}(s \rightarrow s')\right]$ as the expected traveling time, which corresponds to the value function of policy $\pi$ in a stochastic shortest-path setting (SSP,~\cite[][Sect.\,3]{bertsekas1995dynamic}) with initial state $s$, goal state $s'$ and unit cost function. Note that we have $v_{\pi}(s \rightarrow s') = + \infty$ when the policy $\pi$ does not reach $s'$ from $s$ with probability 1. Furthermore, for any subset $\mathcal{S}' \subseteq \mathcal{S}$
and any state $s$, we denote by
\begin{align*}
V^{\star}_{\mathcal{S}'}(s_0 \rightarrow s) := \min_{\pi \in \Pi(\mathcal{S}')} v_{\pi}(s_0 \rightarrow s),
\end{align*}
\vspace{-0.05in}%
the length of the shortest path to $s$, restricted to policies resetting to $s_0$ from any state outside $\mathcal{S}'$.
\end{definition}

The objective of the learning agent is to \textit{control efficiently} the environment in the \textit{vicinity} of $s_0$. 
We say that a state $s$ is controlled if the agent can reliably navigate to it from $s_0$, that is, there exists an effective \textit{goal-conditioned policy} --- i.e., a \textit{shortest-path policy} --- from $s_0$ to $s$. 

\begin{definition}[$L$-controllable states]
Given a reference state $s_0$, we say that a state $s$ is \textit{$L$-controllable} if there exists a policy $\pi$ such that $v_{\pi}(s_0 \rightarrow s) \leq L$. The set of \textit{$L$-controllable states} is then
\begin{align}\label{eq_S_L}
\mathcal{S}_L := \{ s \in \mathcal{S} : \min_{\pi \in \Pi} v_{\pi}(s_0 \rightarrow s) \leq L \}.
\end{align}
\end{definition}
\vspace{-0.05in}
We illustrate the concept of controllable states in Fig.\,\ref{fig_star} for $L=3$. Interestingly, in the right figure, the black states are not $L$-controllable. In fact, there is no policy that can directly choose which one of the black states to reach. On the other hand, the red state, despite being in some sense \textit{further} from $s_0$ than the black states, \textit{does} belong to $S_L$. In general, there is a crucial difference between the existence of a \textit{random} realization where a state $s$ is reached from $s_0$ in less than $L$ steps (i.e., black states) and the notion of $L$-\textit{controllability}, which means that there exists a policy that consistently reaches the state in a number of steps less or equal than $L$ on average (i.e., red state). This explains the choice of the term \emph{controllable} over \emph{reachable}, since a state $s$ is often said to be reachable if there is a policy $\pi$ with a non-zero probability to eventually reach it, which is a  weaker requirement.

\tikzset{
  pics/carc/.style args={#1:#2:#3}{
    code={
      \draw[pic actions] (#1:#3) arc(#1:#2:#3);
    }
  }
}

\newcommand{\stargraph}[2]{
\resizebox{4.5cm}{4.5cm}{%
\begin{tikzpicture}
    \node[circle,draw=black,fill=white] at (360:0mm) (center) {};
    \foreach \n in {1,...,#1}{

        \node[circle,draw=black] at ({\n*360/#1}:#2*5 cm) (n\n) {};
        \draw (center)--(n\n);
				\node[circle,draw=black,fill=black] at ({\n*360/#1}:#2*5 cm) (n\n) {};

        \node[circle,draw=black] at ({\n*360/#1}:#2*4 cm) (n\n) {};
        \draw (center)--(n\n);
				\node[circle,draw=black,fill=black] at ({\n*360/#1}:#2*4 cm) (n\n) {};

        \node[circle,fill=red!50] at ({\n*360/#1}:#2*3 cm) (n\n) {};
        \draw (center)--(n\n);

        \node[circle,fill=red!50] at ({\n*360/#1}:#2*2 cm) (n\n) {};
        \draw (center)--(n\n);

        \node[circle,fill=red!50] at ({\n*360/#1}:#2*1 cm) (n\n) {};
        \draw (center)--(n\n);

        \pic[red,thick,dashed]{carc=-180:180:#2*3.4cm};

    }
\end{tikzpicture}
}
}

\begin{figure}[t]
    \begin{minipage}{.34\linewidth}
    \stargraph{8}{0.6}
\end{minipage}%
\begin{minipage}{0.47\linewidth}
\caption{Two environments where the starting state $s_0$ is in white. \emph{Left:} Each transition between states is deterministic and depicted with an edge. \emph{Right:} Each transition from $s_0$ to the first layer is \textit{equiprobable} and the transitions in the successive layers are deterministic. If we set $L=3$, then the states belonging to $\mathcal{S}_L$ are colored in red. As the right figure illustrates, $L$-controllability is not necessarily linked to a notion of distance between states and an $L$-controllable state may be achieved by traversing states that are not $L$-controllable themselves.}
\label{fig_star}
\end{minipage}%
\begin{minipage}{0.19\linewidth}
    \flushright
	\begin{tikzpicture}[thick,scale=0.9]
	\node[circle,draw=black,fill=white] at (1,1.5) (1) {}; 
    \node[circle,fill=black] at (0,0.5) (2) {}; 
    \node[circle,fill=black] at (1,0.5) (3) {}; 
    \node[circle,fill=black] at (2,0.5) (4) {}; 
    \node[circle,fill=black] at (0,-0.5) (5) {}; 
    \node[circle,fill=black] at (1,-0.5) (6) {}; 
    \node[circle,fill=black] at (2,-0.5) (7) {}; 
    \node[circle,fill=red!50 ] at (1,-1.5) (8) {}; 
	\begin{scope}[>={Stealth[black]},
	every node/.style={fill=white,circle},
	every edge/.style={draw=gray, thick},
	every loop/.style={draw=gray, thick, min distance=5mm,looseness=5}]
	\path[]
	(1) [->,thick] edge[] node[scale=0.001, text width = 0mm] {} (2)
	(1) [->,thick] edge[] node[scale=0.001, text width = 0mm] {} (3)
	(1) [->,thick] edge[] node[scale=0.001, text width = 0mm] {} (4)
	(2) [->,thick] edge[] node[scale=0.001, text width = 0mm] {} (5)
	(3) [->,thick] edge[] node[scale=0.001, text width = 0mm] {} (6)
	(4) [->,thick] edge[] node[scale=0.001, text width = 0mm] {} (7)
    (5) [->,thick] edge[] node[scale=0.001, text width = 0mm] {} (8)
    (6) [->,thick] edge[] node[scale=0.001, text width = 0mm] {} (8)
    (7) [->,thick] edge[] node[scale=0.001, text width = 0mm] {} (8);
	\end{scope}
	\end{tikzpicture}
     \end{minipage}
\end{figure}
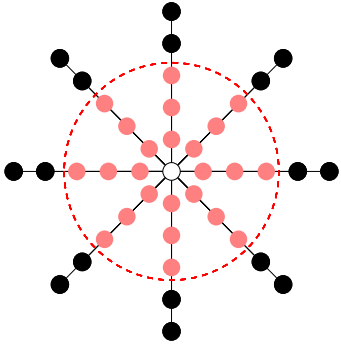

Unfortunately, Lim and Auer~\cite{lim2012autonomous} showed that in order to discover all the states in $\mathcal{S}_L$, the learner may require a number of exploration steps that is \textit{exponential} in $L$ or $\abs{\mathcal{S}_L}$. Intuitively, this negative result is due to the fact that the minimum in Eq.\,\ref{eq_S_L} is over the set of all possible policies, including those that may traverse states that are not in $\mathcal{S}_L$.\footnote{We refer the reader to~\cite[][Sect.\,2.1]{lim2012autonomous} for a more formal and complete characterization of this negative result.} Hence, we similarly constrain the learner to focus on the set of \textit{incrementally controllable} states.

\begin{definition}[Incrementally controllable states $\mathcal{S}_L^{\rightarrow}$]
Let $\prec$ be some partial order on $\mathcal{S}$. The set $\mathcal{S}_L^{\prec}$ of states controllable in $L$ steps w.r.t.\ $\prec$ is defined inductively as follows. The initial state $s_0$ belongs to $\mathcal{S}_L^{\prec}$ by definition and if there exists a policy $\pi$ restricted on $\{ s' \in \mathcal{S}_L^{\prec}: s' \prec s \}$ with $v_{\pi}(s_0 \rightarrow s) \leq L$, then $s \in \mathcal{S}_L^{\prec}$.
The set $\mathcal{S}_L^{\rightarrow}$ of incrementally $L$-controllable states is defined as $\mathcal{S}_L^{\rightarrow} := \cup_{\prec} \mathcal{S}_L^{\prec}$, where the union is over all possible partial orders.
\label{definition_incremental_set}
\end{definition}
By way of illustration, in Fig.\,\ref{fig_star} for $L=3$, it holds that $\mathcal{S}_L^{\rightarrow} = \mathcal{S}_L$ in the left figure, whereas $\mathcal{S}_L^{\rightarrow} = \{ s_0 \} \neq \mathcal{S}_L$ in the right figure. Indeed, while the red state is $L$-controllable, it requires traversing the black states, which are not $L$-controllable. 


\subsection{\DC Objectives}\label{ssec:dc.objectves}

We are now ready to formalize two alternative objectives for \textit{Autonomous eXploration} (\DC) in MDPs.



\begin{definition}[\DC sample complexity]\label{definition_exploration_bound}
Fix any length $L \geq 1$, error threshold $\epsilon > 0$ and confidence level $\delta \in (0,1)$. 
The sample complexities $\mathcal{C}_{\DCL}(\mathfrak{A},L,\epsilon,\delta)$ and $\mathcal{C}_{\small\textsc{AX}^{\star}}(\mathfrak{A},L,\epsilon,\delta)$ are defined as the number of time steps required by a learning algorithm $\mathfrak{A}$ to identify a set $\mathcal{K} \supseteq \mathcal{S}_L^{\rightarrow}$ such that with probability at least $1 - \delta$, it has learned a set of policies $\{ \pi_s \}_{s \in \mathcal{K}}$ that respectively verifies 
the following \DC requirement
\begin{itemize}[itemsep=0.1pt,topsep=0pt]
    \item[\textup{(\DCL)}] ~ $ \forall s \in \mathcal{K}, v_{\pi_s}(s_0 \rightarrow s) \leq L + \epsilon$,
    \item[\textup{(\DCstar)}] ~ $\forall s \in \mathcal{K}, v_{\pi_s}(s_0 \rightarrow s) \leq V^{\star}_{\mathcal{S}_L^{\rightarrow}}(s_0 \rightarrow s) + \epsilon.$
\end{itemize}

%
\end{definition}

Designing agents satisfying the objectives defined above introduces critical difficulties w.r.t.\,standard goal-directed learning in RL. First, the agent has to find accurate policies for a set of goals (i.e., all incrementally $L$-controllable states) and not just for one specific goal. On top of this, the set of desired goals itself (i.e., the set $\mathcal{S}_L^{\rightarrow}$) is \textit{unknown} in advance and has to be estimated online. Specifically, \DCL is the original objective introduced in~\cite{lim2012autonomous} and it requires the agent to discover all the incrementally $L$-controllable states as fast as possible.\footnote{Note that we translated in  the condition in \cite{lim2012autonomous} of a relative error of $L \epsilon$ to an absolute error of $\epsilon$, to align it with the common formulation of sample complexity in RL.\label{footref_absolute_relative_error}} At the end of the learning process, for each state $s\in\mathcal{S}_L^{\rightarrow}$ the agent should return a policy that can reach $s$ from $s_0$ in at most $L$ steps (in expectation). Unfortunately, this may correspond to a rather poor performance in practice. Consider a state $s\in\mathcal{S}_L^{\rightarrow}$ such that $V^{\star}_{\mathcal{S}_L^{\rightarrow}}(s_0 \rightarrow s) \ll L$, i.e., the shortest path between $s_0$ to $s$ following policies restricted on $\mathcal{S}_L^{\rightarrow}$ is much smaller than $L$. Satisfying \DCL only guarantees that a policy reaching $s$ in $L$ steps is found. On the other hand, objective \DCstar is more demanding, as it requires learning a near-optimal shortest-path policy for each state in $\mathcal{S}_L^{\rightarrow}$. Since $V^{\star}_{\mathcal{S}_L^{\rightarrow}}(s_0 \rightarrow s) \leq L$ and the gap between the two quantities may be arbitrarily large, especially for states close to $s_0$ and far from the fringe of $\mathcal{S}_L^{\rightarrow}$, \DCstar is a significantly tighter objective than \DCL and it is thus preferable in practice.

We say that an exploration algorithm solves the \DC problem if its sample complexity $\mathcal{C}_{\DC}(\mathfrak{A},L,\epsilon,\delta)$ in Def.\,\ref{definition_exploration_bound} is polynomial in $\abs{\mathcal{K}}$, $A$, $L,$ $\epsilon^{-1}$ and $\log(S)$. Notice that requiring a logarithmic dependency on the size of $\mathcal{S}$ is crucial but nontrivial, since the overall state space may be large and we do not want the agent to waste time trying to reach states that are not $L$-controllable. The dependency on the (algorithmic-dependent and random) set $\mathcal{K}$ can be always replaced using the upper bound $\abs{\mathcal{K}} \leq \abs{\mathcal{S}^{\rightarrow}_{L+\epsilon}}$, which is implied with high probability by both \DCL and \DCstar conditions. Finally, notice that the error threshold $\epsilon > 0$ has a two-fold impact on the performance of the algorithm. First, $\epsilon$ defines the largest set $\mathcal{S}^{\rightarrow}_{L + \epsilon}$ that could be returned by the algorithm: the larger $\epsilon$, the bigger the set. Second, as $\epsilon$ increases, the quality (in terms of controllability and navigational precision) of the output policies worsens w.r.t.\,the shortest-path policy restricted on $\mathcal{S}_L^{\rightarrow}$. 


\section{The \ALGOtitle Algorithm}


The algorithm \ALGO\,--- for \texttt{Discover and Control} --- is detailed in Alg.\,\ref{algorithm_SSP_generative_model}.
It maintains a set $\mathcal{K}$ of \say{controllable} states and a set $\mathcal{U}$ of states that are considered \say{uncontrollable} \textit{so far}. A state $s$ is tagged as controllable when a policy to reach $s$ in at most $L + \epsilon$ steps (in expectation from $s_0$) has been found with high confidence, and we denote by $\pi_s$ such policy. The states in $\mathcal{U}$ are states that have been discovered as potential members of $\mathcal{S}_L^{\rightarrow}$, but the algorithm has yet to produce a policy to control any of them in less than $L + \epsilon$ steps. The algorithm stores an estimate of the transition model and it proceeds through rounds, which are indexed by $k$ and incremented whenever a state in $\mathcal{U}$ gets transferred to the set $\mathcal{K}$, i.e., when the transition model reaches a level of accuracy sufficient to compute a policy to control one of the states encountered before. We denote by $\mathcal{K}_k$ (resp.\,$\mathcal{U}_k$) the set of controllable (resp.\,uncontrollable) states at the beginning of round $k$. \ALGO stops at a round $K$ when it can confidently claim that all the remaining states outside of $\mathcal{K}_K$ cannot be $L$-controllable.

At each round, the algorithm uses all samples observed so far to build an estimate of the transition model denoted by $\wh{p}(s'|s,a) = N(s,a,s') / N(s,a)$, where $N(s,a)$ and $N(s,a,s')$ are counters for state-action and state-action-next state visitations. Each round is divided into two phases. The first is a \textit{sample collection} phase. At the beginning of round $k$, the agent collects additional samples until $n_k := \phi(\mathcal{K}_k)$ samples are available at each state-action pair in $\mathcal{K}_k \times \mathcal{A}$ (step \ding{172}). A key challenge lies in the careful (and adaptive) choice of the allocation function $\phi$, which we report in the statement of Thm.\,\ref{thm:sample.comlexity} (see Eq.\,\ref{budget_modest.detail} in App.\,\ref{app_state_transfer} for its exact definition). Importantly, the incremental construction of $\mathcal{K}_k$ entails that sampling at each state $s \in \mathcal{K}_k$ can be done efficiently. In fact, for all $s\in\mathcal{K}_k$ the agent has already confidently learned a policy $\pi_s$ to reach $s$ in at most $L + \epsilon$ steps on average (see how such policy is computed in the second phase). \jt{The generation of transitions $(s,a,s')$ for $(s,a)\in \mathcal{K}_k \times \mathcal{A}$ achieves two objectives at once. First, it serves as a discovery step, since all observed next states $s'$ not in $\mathcal{U}_k$ are added to it --- in particular this guarantees sufficient exploration at the fringe (or border) of the set $\mathcal{K}_k$. 
Second, it improves the accuracy of the model $p$ in the states in~$\mathcal{K}_k$, which is essential in computing near-optimal policies and thus fulfilling the \DCstar condition.}



\newlength{\textfloatsepsave}
\setlength{\textfloatsepsave}{\textfloatsep}
\setlength{\textfloatsep}{0pt}

\begin{algorithm}[t!]
	\begin{small}
		\DontPrintSemicolon
		\caption{Algorithm \ALGO}
		\label{algorithm_SSP_generative_model}
        \KwIn{Actions $\mathcal{A}$, initial state $s_0$, confidence parameter $\delta \in (0,1)$, error threshold $\epsilon > 0$, $L \geq 1$ and (possibly adaptive) allocation function $\phi: \mathcal{P}(\mathcal{S}) \rightarrow \mathbb{N}$ (where $\mathcal{P}(\mathcal{S})$ denotes the power set of $\cS$).}
        Initialize $k := 0$, $\mathcal{K}_0 := \{ s_0 \}$, $\mathcal{U}_0 := \{ \}$ and a restricted policy $\pi_{s_0} \in \Pi(\mathcal{K}_0)$.\;
		Set $\epsilon := \min \{ \epsilon, 1 \}$  \label{line_epsilon_min}
        and  $\texttt{continue} := \texttt{True}$.\;
		\While{$\textup{\texttt{continue}}$}{
			Set $k \mathrel{+}= 1$. \CommentSty{\small{//new round}} \\
			\tcp{\small{\textbf{\ding{172} Sample collection on $\mathcal{K}$}}}
            For each $(s,a) \in \mathcal{K}_k \times \mathcal{A}$, execute policy $\pi_{s}$ until the total number of visits $N_k(s,a)$ to $(s,a)$ satisfies  $N_k(s,a) \geq n_k := \phi(\mathcal{K}_k)$.
            For each $(s,a) \in \mathcal{K}_k \times \mathcal{A}$, add $s' \sim p(\cdot|s,a)$ to $\mathcal{U}_k$ if $s' \notin \mathcal{K}_k$.\;
			\tcp{\small{\textbf{\ding{173} Restriction of candidate states $\mathcal{U}$}}}
            Compute transitions  $\wh p_k(s'|s,a)$
            and $\mathcal{W}_k := \Big\{ s' \in \mathcal{U}_k: \exists (s,a) \in \mathcal{K}_k \times \mathcal{A}, \wh{p}_k(s' \vert s,a) \geq \frac{1 - \epsilon / 2}{L} \Big\}\cdot$ \;
			\If{$\mathcal{W}_k$ \textup{is empty}}
			{Set $\texttt{continue} := \texttt{False}$. ~\CommentSty{\small{//condition \StopOne}}}
			\Else{
				\tcp{\small{\textbf{\ding{174} Computation of the optimistic policies on $\mathcal{K}$}}}
				\For{\textup{each state} $s' \in \mathcal{W}_k$}{
                        Compute $(\wt{u}_{s'}, \wt{\pi}_{s'}) := \OVISSPmath(\cK_k, \cA, s', N_k, \frac{\epsilon}{6L}),$ see Alg.\,\ref{alg:OVISSP.app} in App.\,\ref{app:app_full_proof.ovi}.\;
                }
                Let $s^\dagger := \argmin_{s \in \mathcal{W}_k} \wt{u}_s(s_0)$ and $\wt{u}^{\dagger} := \wt{u}_{s^\dagger}(s_0)$.\;
				\If{$\wt{u}^{\dagger} > L$}
				{Set $\texttt{continue} := \texttt{False}$. ~\CommentSty{\small{//condition \StopTwo}}
				}
				\Else
				{
					\tcp{\small{\textbf{\ding{175} State transfer from $\mathcal{U}$ to $\mathcal{K}$}}}
                    Set $\mathcal{K}_{k+1} := \mathcal{K}_k \cup \{ s^{\dagger} \}$,~ $\mathcal{U}_{k+1} := \mathcal{U}_k \setminus \{ s^{\dagger} \}$ and $\pi_{s^{\dagger}} := \wt{\pi}_{s^\dagger}$.
				}
			}
		}
		\tcp{\small{\textbf{\ding{176} Policy consolidation:\,computation on the final set $\mathcal{K}$}}}
		Set $K := k$.\\
		\For{\textup{each state} $s \in \mathcal{K}_{K}$}{
                Compute $(\wt{u}_s, \wt{\pi}_s) := \OVISSPmath(\mathcal{K}_{K}, \cA, s, N_K, \frac{\epsilon}{6L})$.\;
		}
		\textbf{Output:} the states $s$ in $\mathcal{K}_{K}$ and their corresponding policy $\pi_s := \wt{\pi}_s$.
	\end{small}
\end{algorithm}

\setlength{\textfloatsep}{\textfloatsepsave}

The second phase does not require interacting with the environment and it focuses on the \textit{computation of optimistic policies}. The agent begins by significantly restricting the set of candidate states in each round to alleviate the computational complexity of the algorithm. Namely, among all the states in $\mathcal{U}_k$, it discards those that do not have a high probability of belonging to $\mathcal{S}_L^{\rightarrow}$ by considering a restricted set $\mathcal{W}_k \subseteq \mathcal{U}_k$ (step \ding{173}). In fact, if the estimated probability $\wh{p}_k$ of reaching a state $s\in\mathcal{U}_k$ from any of the controllable states in $\mathcal{K}_k$ is lower than $(1-\epsilon/2)/L$, then no shortest-path policy restricted on $\mathcal{K}_k$ could get to $s$ from $s_0$ in less than $L+\epsilon$ steps on average. Then for each state $s'$ in $\mathcal{W}_k$, \ALGO computes an optimistic policy restricted on $\mathcal{K}_k$ to reach $s'$. Formally, for any candidate state $s' \in \mathcal{W}_k$, we define the induced stochastic shortest path (SSP) MDP $M'_{k}$ with goal state $s'$ as follows.

\begin{definition}
        We define the SSP-MDP $M'_{k} := \langle \mathcal{S}, \mathcal{A}'_k(\cdot), c'_{k}, p'_k\rangle$ with goal state $s'$, where the action space is such that $\cA'_k(s) = \cA$ for all  $s\in\mathcal{K}_k$ and $\cA'_k(s) = \{\RESET\}$ otherwise (i.e., we focus on policies restricted on $\cK_k$). The cost function is such that for all $a \in \mathcal{A}$, $c'_{k}(s',a) = 0$, and for any $s \neq s'$, $c'_{k}(s,a) = 1$. The transition model is $p'_{k}(s' \vert s',a) = 1$ and $p'_{k}(\cdot \vert s,a) = p(\cdot \vert s,a)$ otherwise.\footnote{In words, all actions at states in $\mathcal{K}_k$ behave exactly as in $M$ and suffer a unit cost, in all states outside $\cK_k$ only the reset action to $s_0$ is available with a unit cost, and all actions at the goal $s'$ induce a zero-cost self-loop.}
	\label{definition_induced_mdp}
\end{definition}

The solution of $M'_k$ is the shortest-path policy from $s_0$ to $s'$ restricted on $\cK_k$. Since $p'_k$ is unknown, \ALGO cannot compute the exact solution of $M'_k$, but instead, it executes optimistic value iteration (\OVISSP) for SSP~\cite{tarbouriech2019no, cohen2020near} to obtain a value function $\wt{u}_{s'}$ and its associated greedy policy $\wt{\pi}_{s'}$ restricted on $\cK_k$ (see App.\,\ref{app:app_full_proof.ovi} for more details).

The agent then chooses a candidate goal state $s^{\dagger}$ for which the value $\wt{u}^\dagger:=\wt{u}_{s^\dagger}(s_0)$ is the smallest. This step can be interpreted as selecting the optimistically most promising new state to control. Two cases are possible. If $\wt{u}^{\dagger} \leq L$, then $s^{\dagger}$ is added to $\mathcal{K}_k$ (step \ding{175}), since the accuracy of the model estimate on the state-action space $\mathcal{K}_k \times \mathcal{A}$ guarantees that the policy $\wt{\pi}_{s^\dagger}$ is able to reach the state $s^{\dagger}$ in less than $L+\epsilon$ steps in expectation with high probability (i.e., $s^{\dagger}$ is incrementally $(L+\epsilon)$-controllable). Otherwise, we can guarantee that $\mathcal{S}_L^{\rightarrow} \subseteq \mathcal{K}_k$ with high probability. In the latter case, the algorithm terminates and, using the current estimates of the model, it recomputes an optimistic shortest-path policy $\pi_s$ restricted on the final set $\mathcal{K}_{K}$  for each state $s \in \mathcal{K}_{K}$ (step \ding{176}). \jt{This policy consolidation step is essential to identify near-optimal policies restricted on the final set $\mathcal{K}_{K}$ (and thus on $\mathcal{S}_L^{\rightarrow}$): indeed the expansion of the set of the so far controllable states may alter and refine the optimal goal-reaching policies restricted on it (see App.\,\ref{app_objectives}).}

\jt{\textbf{Computational Complexity.} Note that algorithmically, we do not need to define $M'_k$ (Def.\,\ref{definition_induced_mdp}) over the whole state space $\cS$ as we can limit it to $\mathcal{K}_k \cup \{ s' \}$, i.e., the candidate state $s'$ and the set $\mathcal{K}_k$ of so far controllable states. As shown in Thm.\,\ref{theorem_bound_UCSSPGM}, this set can be significantly smaller than~$\cS$. In particular this implies that the computational complexity of the value iteration algorithm used to compute the optimistic policies is independent from $S$ (see App.\,\ref{subsection_computational_complexities} for more details).}



\section{Sample Complexity Analysis of \ALGOtitle}
\label{sect_analysis}

We now present our main result: a sample complexity guarantee for \ALGO for the \DCstar objective, which directly implies that \DCL is also satisfied.


\begin{theorem}\label{thm:sample.comlexity}
    There exists an absolute constant $\alpha > 0$ such that for any $L \geq 1$, $\epsilon \in (0,1],$ and $\delta \in (0,1)$, if we set the allocation function $\phi$ as
    \begin{align}
        \phi : \mathcal{X} \rightarrow \alpha \cdot \left( \frac{L^4 \wh{\Theta}(\mathcal{X})}{\epsilon^2} \log^2 \left( \frac{L S A }{\epsilon \delta} \right) + \frac{L^2 \abs{\mathcal{X}}}{ \epsilon} \log\left( \frac{L S A }{\epsilon \delta} \right) \right)\CommaBin
        \label{allocation_function}
    \end{align}
    with $\wh{\Theta}(\mathcal{X}) := \max_{(s,a) \in \mathcal{X} \times \mathcal{A}} \big(\sum_{s' \in \mathcal{X}} \sqrt{ \wh{p}(s' \vert s,a)(1 - \wh{p}(s' \vert s,a))}\big)^2$,
    then the algorithm \ALGO (Alg.\,\ref{algorithm_SSP_generative_model}) satisfies the following sample complexity bound for \DCstar
    \begin{align}\label{eq:algo.bound}
        \mathcal{C}_{\small\textsc{AX}^{\star}}(\ALGO, L, \epsilon, \delta) = \wt{O}\left( \frac{L^5 \Gamma_{\Lpluseps} S_{\Lpluseps} A}{\epsilon^2} + \frac{L^3 S_{\Lpluseps}^2 A}{\epsilon}\right)\CommaBin
    \end{align}
where $S_{\Lpluseps} := \abs{\mathcal{S}_{L+\epsilon}^{\rightarrow}}$ and
\begin{align*}
    \Gamma_{\Lpluseps} := \max_{(s,a) \in \mathcal{S}_{\Lpluseps}^{\rightarrow} \times \mathcal{A}} \norm{\{ p(s' \vert s, a)\}_{s' \in \mathcal{S}_{\Lpluseps}^{\rightarrow}}}_0 \leq S_{\Lpluseps}
\end{align*}
is the maximal support of the transition probabilities $p(\cdot \vert s, a)$ \textit{restricted} to the set $\mathcal{S}_{\Lpluseps}^{\rightarrow}$.
\label{theorem_bound_UCSSPGM}
\end{theorem}

Given the definition of \DCstar, Thm.\,\ref{theorem_bound_UCSSPGM} implies that \ALGO \textbf{1)} terminates after $\mathcal{C}_{\small\textsc{AX}^{\star}}(\ALGO, L, \epsilon, \delta)$ time steps, \textbf{2)} discovers a set of states $\mathcal{K} \supseteq \mathcal{S}_L^{\rightarrow}$ with $\abs{\mathcal{K}} \leq S_{\Lpluseps}$, \textbf{3)} and for each $s \in \mathcal{K}$ outputs a policy~$\pi_s$ which is $\epsilon$-optimal w.r.t.\,policies restricted on $\mathcal{S}_L^{\rightarrow}$, i.e., $v_{\pi_s}(s_0 \rightarrow s) \leq V^{\star}_{\mathcal{S}_{L}^{\rightarrow}}(s_0 \rightarrow s) + \epsilon$. \jt{Note that Eq.\,\ref{eq:algo.bound} displays only a \textit{logarithmic} dependency on $S$, the total number of states. This property on the sample complexity of \ALGO, along with its $S$-independent computational complexity, is significant when the state space $\mathcal{S}$ grows large w.r.t.\,the unknown set of interest $\mathcal{S}_L^{\rightarrow}$.}

\subsection{Proof Sketch of Theorem \ref{theorem_bound_UCSSPGM}}
\label{proof_sketch}

While the complete proof is reported in App.\,\ref{app_full_proof}, we now provide the main intuition behind the result. 




\vspace{-0.04in}

\paragraph{State Transfer from $\mathcal{U}$ to $\mathcal{K}$ (step \ding{175}).}

Let us focus on a round $k$ and a state $s^{\dagger} \in \mathcal{U}_k$ that gets added to $\mathcal{K}_k$. For clarity we remove in the notation the round $k$, goal state $s^{\dagger}$ and starting state $s_0$. We denote by $v$ and $\wt{v}$ the value functions of the candidate policy $\wt{\pi}$ in the true and optimistic model respectively, and by $\wt{u}$ the quantity w.r.t.\,which $\wt{\pi}$ is optimistically greedy. We aim to prove that $s^{\dagger} \in \mathcal{S}_{L+\epsilon}^{\rightarrow}$ (with high probability). The main chain of inequalities underpinning the argument is
\vspace{-0.02in}
\begin{align}
    v \leq \abs{v - \wt{v}} + \wt{v} \myineeqa \frac{\epsilon}{2} + \wt{v} \myineeqb \frac{\epsilon}{2} + \wt{u} + \frac{\epsilon}{2} \myineeqc L + \epsilon,
    \label{chain_inequalities}
\end{align}
where (c) is guaranteed by algorithmic construction and (b) stems from the chosen level of value iteration accuracy. Inequality (a) has the flavor of a simulation lemma for SSP, by relating the shortest-path value function of a same policy between two models (the true one and the optimistic one). Importantly, when restricted to $\mathcal{K}$ these two models are close in virtue of the algorithmic design which enforces the collection of a minimum amount of samples at each state-action pair of $\mathcal{K} \times \mathcal{A}$, denoted by $n$. Specifically, we obtain that
\vspace{-0.02in}
\begin{align*}
\abs{v - \wt{v}} = \wt{O}\Big( \sqrt{\frac{ L^4 \Gamma_{\mathcal{K}}}{n}} + \frac{L^2 \abs{\mathcal{K}}}{n} \Big),  \quad \quad \textrm{with} \quad \Gamma_{\mathcal{K}} := \max_{(s,a) \in \mathcal{K} \times \mathcal{A}} \norm{\{ p(s' \vert s, a)\}_{s' \in \mathcal{K}}}_0 \leq \abs{\mathcal{K}}.
\end{align*}
Note that $\Gamma_{\mathcal{K}}$ is the branching factor restricted to the set $\mathcal{K}$. Our choice of~$n$ (given in Eq.\,\ref{allocation_function}) is then dictated to upper bound the above quantity by~$\epsilon / 2$ in order to satisfy inequality~(a). \jt{Let us point out that, interestingly yet unfortunately, the structure of the problem does not appear to allow for technical variance-aware improvements seeking to lower the value of $n$ prescribed above (indeed the \DC framework requires to analytically encompass the uncontrollable states $\mathcal{U}$ into a single meta state with higher transitional uncertainty, see App.\,\ref{app_full_proof} for details).}

\vspace{-0.04in}

\paragraph{Termination of the Algorithm.}

Since $\mathcal{S}_L^{\rightarrow}$ is \textit{unknown}, we have to ensure that none of the states in $\mathcal{S}_L^{\rightarrow}$ are \say{missed}. As such, we prove that with overwhelming probability, we have $\mathcal{S}_L^{\rightarrow} \subseteq \mathcal{K}_K$ when the algorithm terminates at a round denoted by $K$. There remains to justify the final near-optimal guarantee w.r.t.\,the set of policies $\Pi(\mathcal{S}_L^{\rightarrow})$. Leveraging that step \ding{176} recomputes the policies $(\pi_s)_{s \in \mathcal{K}_K}$ on the final set $\mathcal{K}_K$, we establish the following chain of inequalities
\vspace{-0.02in}
\begin{align}
    v \leq \abs{v - \wt{v}} + \wt{v} \myineeqa \frac{\epsilon}{2} + \wt{v} \myineeqb \frac{\epsilon}{2} + \wt{u} + \frac{\epsilon}{2} \myineeqc V^{\star}_{\mathcal{K}_K} + \epsilon \myineeqd V^{\star}_{\mathcal{S}_L^{\rightarrow}} + \epsilon,
    \label{chain_inequalities_2}
\end{align}
where (a) and (b) are as in Eq.\,\ref{chain_inequalities}, (c) leverages optimism and (d) stems from the inclusion $\mathcal{S}_L^{\rightarrow} \subseteq \mathcal{K}_K$.

\vspace{-0.04in}

\paragraph{Sample Complexity Bound.} The choice of allocation function $\phi$ in Eq.\,\ref{allocation_function} bounds $n_K$ which is the total number of samples required at each state-action pair in $\mathcal{K}_K \times \mathcal{A}$. We then compute a high-probability bound $\psi$ on the time steps needed to collect a given sample, and show that it scales as $\wt{O}(L)$. Since the sample complexity is solely induced by the sample collection phase (step~\ding{172}), it can be bounded by the quantity $\psi \, n_K \abs{\mathcal{K}_K} A$. Putting everything together yields the bound of Thm.\,\ref{theorem_bound_UCSSPGM}.

\subsection{Comparison with \UcbExplore~\cite{lim2012autonomous}}
We start recalling the critical distinction that \ALGO succeeds in tackling problem \DCstar, while \UcbExplore~\cite{lim2012autonomous} fails to do so (see App.\,\ref{app_objectives} for details on the \DC objectives). Nonetheless, in the following we show that even if we restrict our attention to \DCL, for which \UcbExplore is designed, \ALGO yields a better sample complexity in most of the cases. From \cite{lim2012autonomous}, \UcbExplore verifies\footnote{Note that if we replace the error of $\epsilon$ for \DCL with an error of $L\epsilon$ as in \citep{lim2012autonomous}, we recover the sample complexity of $\wt{O}\left(L^3 S_{\Lpluseps} A/\epsilon^3 \right)$ stated in \cite[][Thm.\,8]{lim2012autonomous}.}
\vspace{-0.02in}
\begin{align}
\mathcal{C}_{\DCL}(\UcbExplore,L,\epsilon,\delta) = \wt{O}\left(\frac{L^6 S_{\Lpluseps} A}{\epsilon^3} \right)\cdot
\label{eq_corrected_ucbexplore}
\end{align}
%
Eq.\,\ref{eq_corrected_ucbexplore} shows that the sample complexity of \UcbExplore is linear in $S_{\Lpluseps}$, while for \ALGO the dependency is somewhat worse. In the main-order term $\wt O(1/\epsilon^2)$ of Eq.\,\ref{eq:algo.bound}, the bound depends linearly on $S_{\Lpluseps}$ but also grows with the branching factor $\Gamma_{\Lpluseps}$, which is not the ``global'' branching factor but denotes the 
number of possible next states in $\mathcal{S}_{\Lpluseps}^{\rightarrow}$ starting from $\mathcal{S}_{\Lpluseps}^{\rightarrow}$. While in general we only have $\Gamma_{\Lpluseps} \leq S_{\Lpluseps}$, in many practical domains (e.g., robotics, user modeling), each state can only transition to a small number of states, i.e., we often have $\Gamma_{\Lpluseps} = O(1)$ as long as the dynamics is not too \say{chaotic}. While \ALGO does suffer from a quadratic dependency on $S_{\Lpluseps}$ in the second term of order $\wt O(1/\epsilon)$, we notice that for any $S_{\Lpluseps} \leq L^3 \epsilon^{-2}$ the bound of \ALGO is still preferable. Furthermore, since for $\epsilon\rightarrow 0$, $S_{\Lpluseps}$ tends to $S_L$, the condition is always verified for small enough $\epsilon$.

Compared to \ALGO, the sample complexity of \UcbExplore is worse in both $\epsilon$ and $L$. 
As stressed in Sect.\,\ref{ssec:dc.objectves}, the better dependency on $\epsilon$ both improves the quality of the output goal-reaching policies as well as reduces the number of incrementally $(L+\epsilon)$-controllable states returned by the algorithm. It is interesting to investigate why the bound of~\cite{lim2012autonomous} (Eq.\,\ref{eq_corrected_ucbexplore}) inherits a $\wt{O}(\epsilon^{-3})$ dependency. As reviewed in App.\,\ref{app_algo_ucbexplore}, \UcbExplore alternates between two phases of state discovery and policy evaluation. The optimistic policies computed by \UcbExplore solve a \textit{finite-horizon problem} (with horizon set to $H_{\Ucb}$). However, minimizing the expected time to reach a target state is intrinsically an SSP problem, which is exactly what \ALGO leverages. By computing policies that solve a finite-horizon problem (note that \UcbExplore resets every $H_{\Ucb}$ time steps), \cite{lim2012autonomous} sets the horizon to $H_{\Ucb} := \lceil L + L^2 \epsilon^{-1} \rceil$,
which leads to a policy-evaluation phase with sample complexity scaling as $\wt{O}(H_{\Ucb} \epsilon^{-2}) = \wt{O}(\epsilon^{-3})$. Since the rollout budget of $\wt{O}(\epsilon^{-3})$ is hard-coded into the algorithm, the dependency on~$\epsilon$ of \UcbExplore's sample complexity cannot be improved by a more refined analysis; instead a different algorithmic approach is required such as the one employed by \ALGO.




\subsection{\jt{Goal-Free Cost-Free Exploration on $\mathcal{S}_L^{\rightarrow}$ with \ALGO}}


\jt{A compelling advantage of \ALGO is that it achieves an accurate estimation of the environment's dynamics restricted to the unknown subset of interest $\mathcal{S}_L^{\rightarrow}$.
In contrast to \UcbExplore which needs to restart its sample collection from scratch whenever~$L$,~$\epsilon$ or some transition costs change, \ALGO can thus be \textit{robust} to changes in such problem parameters. At the end of its exploration phase in Alg.\,\ref{algorithm_SSP_generative_model}, \ALGO is able to perform zero-shot planning to solve other tasks restricted on $\mathcal{S}_L^{\rightarrow}$, such as cost-sensitive ones. 
Indeed in the following we show how the \ALGO agent is able to compute an~$\epsilon/c_{\min}$-optimal policy for \textit{any} stochastic shortest-path problem on $\mathcal{S}_L^{\rightarrow}$ with goal state $s \in \mathcal{S}_L^{\rightarrow}$ (i.e.,~$s$ is absorbing and zero-cost) and cost function lower bounded by $c_{\min} > 0$. }

\begin{corollary}\label{cor:cost.dependent} \jt{There exists an absolute constant $\beta > 0$ such that for any $L \geq 1$, $\epsilon \in (0, 1]$ and $c_{\min} \in (0, 1]$ verifying $\epsilon \leq \beta \cdot (L \, c_{\min})$,
with probability at least $1-\delta$, for \emph{whatever} goal state $s \in \mathcal{S}_L^{\rightarrow}$ and \emph{whatever} cost function $c$ in $[c_{\min}, 1]$, \ALGO can 
compute (after its exploration phase, without additional environment interaction) a policy $\wh{\pi}_{s,c}$ whose SSP value function $V_{\wh{\pi}_{s,c}}$ verifies
\begin{align*}
V_{\wh{\pi}_{s,c}}(s_0 \rightarrow s) \leq V^{\star}_{\mathcal{S}_L^{\rightarrow}}(s_0 \rightarrow s) + \frac{\epsilon}{c_{\min}},
\end{align*}
where $V_{\pi}(s_0 \rightarrow s) := \mathbb{E}\left[ \sum_{t=1}^{\tau_{\pi}(s_0 \rightarrow s)} c(s_t, \pi(s_t)) ~\big\vert~ s_1 = s_0 \right]$ is the SSP value function of a policy $\pi$ and $V^{\star}_{\mathcal{S}_L^{\rightarrow}}(s_0 \rightarrow s) := \min_{\pi \in \Pi(\mathcal{S}_L^{\rightarrow})} V_{\pi}(s_0 \rightarrow s)$ is the optimal SSP value function restricted on $\mathcal{S}_L^{\rightarrow}$.}
\end{corollary}


\jt{It is interesting to compare Cor.\,\ref{cor:cost.dependent} with the reward-free exploration framework recently introduced by Jin et al.\,\cite{jin2020reward} in finite-horizon. At a high level, the result in Cor.\,\ref{cor:cost.dependent} can be seen as a counterpart of~\cite{jin2020reward} beyond finite-horizon problems, specifically in the goal-conditioned setting. While the parameter $L$ defines the horizon of interest for \ALGO, resetting after every $L$ steps (as in finite-horizon) would prevent the agent to identify $L$-controllable states and lead to poor performance. This explains the distinct technical tools used: while \cite{jin2020reward} executes finite-horizon no-regret algorithms, \ALGO deploys SSP policies restricted on the set of states that it \say{controls} so far. Algorithmically, both approaches seek to build accurate estimates of the transitions on a specific (unknown) state space of interest: the so-called \say{significant} states within $H$ steps for \cite{jin2020reward}, and the incrementally $L$-controllable states $S_L^{\rightarrow}$ for \ALGO. Bound-wise, the cost-sensitive \DCstar problem inherits the critical role of the minimum cost $c_{\min}$ in SSP problems (see App.\,\ref{app_sample_complexity_SSP_gen_model} and e.g.,~\cite{tarbouriech2019no,cohen2020near,bertsekas2013stochastic}), which is reflected in the accuracy of Cor.\,\ref{cor:cost.dependent} scaling inversely with~$c_{\min}$. Another interesting element of comparison is the dependency on the size of the state space. While the algorithm introduced in~\cite{jin2020reward} is robust w.r.t.\,states that can be reached with very low probability, it still displays a \emph{polynomial} dependency on the total number of states~$S$. On the other hand, \ALGO has only a \emph{logarithmic} dependency on~$S$, while it directly depends on the number of $(L+\epsilon)$-controllable states, which shows that \ALGO effectively adapts to the state space of interest and it ignores all other states. This result is significant since not only $S_{L+\epsilon}$ can be arbitrarily smaller than $S$, but also because the set $\mathcal{S}_{L+\epsilon}^{\rightarrow}$ itself is initially unknown to the algorithm.}

\section{Numerical Simulation}\label{sec:experiment}

In this section, we provide the first evaluation of algorithms in the incremental autonomous exploration setting. In the implementation of both \ALGO and \UcbExplore, we remove the logarithmic and constant terms for simplicity. We also boost the empirical performance of \UcbExplore in various ways, for example by considering confidence intervals derived from the empirical Bernstein inequality (see\,\cite{azar2017minimax}) as opposed to Hoeffding as done in \cite{lim2012autonomous}. We refer the reader to App.\,\ref{app:experiments} for details on the algorithmic configurations and on the environments considered.

We compare the sample complexity empirically achieved by \ALGO and \UcbExplore. Fig.\,\ref{fig:3plots} depicts the time needed to identify all the incrementally $L$-controllable states when $L=4.5$ for different values of~$\epsilon$, on a confusing chain domain. Note that the sample complexity is achieved soon after, when the algorithm can confidently discard all the remaining states as non-controllable (it is reported in Tab.\,\ref{tab:new_results} of App.\,\ref{app:experiments}). We observe that \ALGO outperforms \UcbExplore for any value of~$\epsilon$. In particular, the gap in performance increases as $\epsilon$ decreases, which matches the theoretical improvement in sample complexity from $\wt{O}(\epsilon^{-3})$ for \UcbExplore to $\wt{O}(\epsilon^{-2})$ for \ALGO. On a second environment --- the combination lock problem introduced in~\citep{azar2012dynamic} --- we notice that \ALGO again outperforms \UcbExplore, as shown in App.\,\ref{app:experiments}.

Another important feature of \ALGO is that it targets the tighter objective \DCstar, whereas \UcbExplore is only able to fulfill objective \DCL and may therefore elect suboptimal policies. In App.\,\ref{app:experiments} we show empirically that, as expected theoretically, this directly translates into higher-quality goal-reaching policies recovered by \ALGO.

\begin{figure}[t]
        \centering
        \includegraphics[width=.3\textwidth]{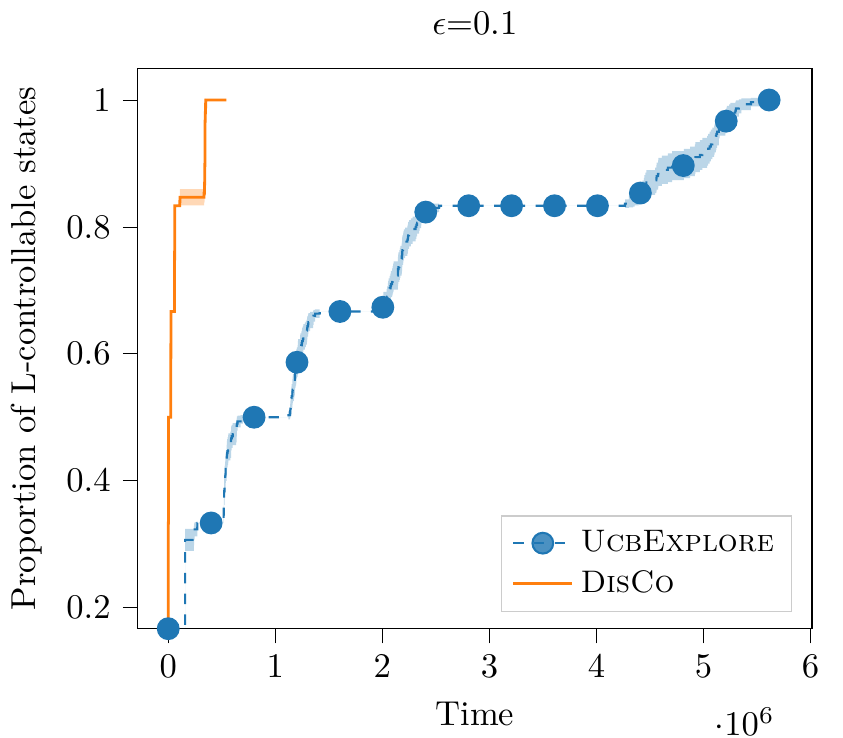}\hfill
        \includegraphics[width=.3\textwidth]{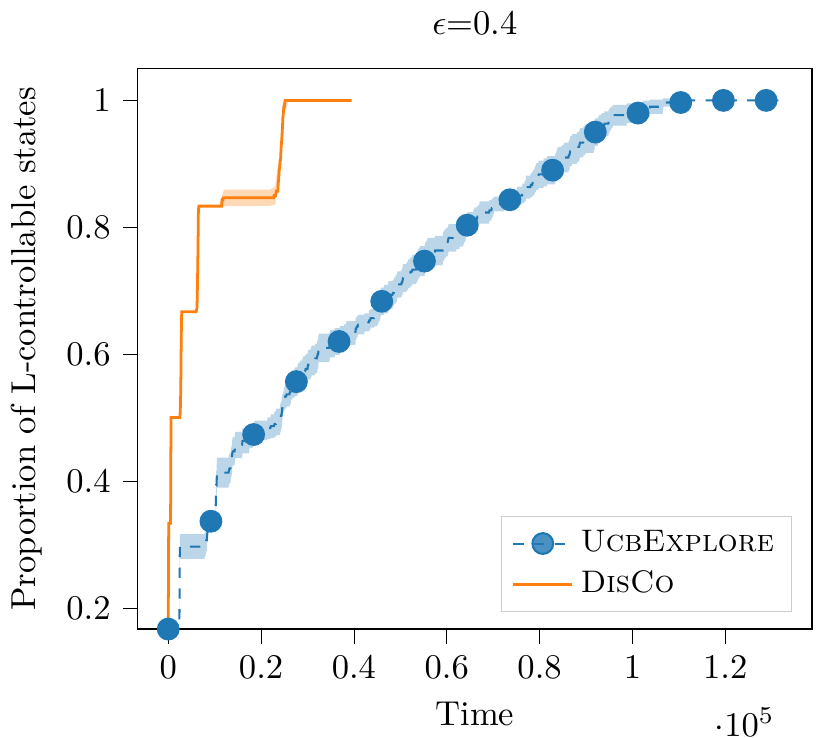}\hfill
        \includegraphics[width=.3\textwidth]{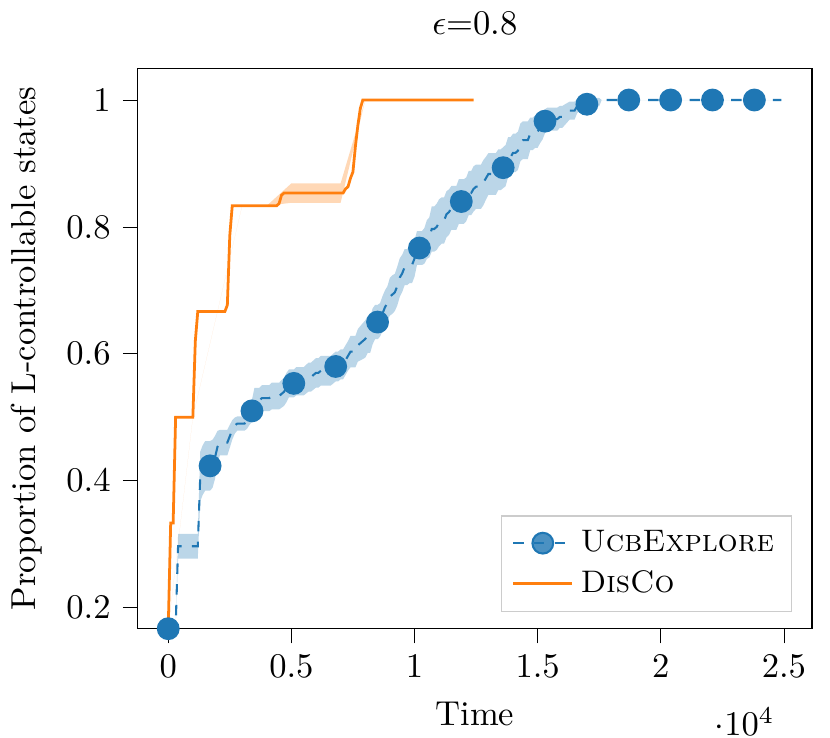}
        \caption{Proportion of the incrementally $L$-controllable states identified by \ALGO and \UcbExplore in a confusing chain domain for $L=4.5$ and $\epsilon \in \{0.1, 0.4, 0.8\}$. Values are averaged over $50$ runs.} 
				\label{fig:3plots}
\end{figure}

\section{Conclusion and Extensions}

\textbf{Connections to existing deep-RL  methods.} While we primarily focus the analysis of \ALGO in the tabular case, we believe that the formal definition of \DC problems and the general structure of \ALGO may also serve as a theoretical grounding of many recent approaches to unsupervised exploration. For instance, it is interesting to draw a parallel between \ALGO and the ideas behind Go-Explore~\cite{ecoffet2019go}. Go-Explore similarly exploits the following principles: (1) remember states that have previously been visited, (2) first return to a promising state (without exploration), (3) then explore from it. Go-Explore assumes that the world is deterministic and resettable, meaning that one can reset the state of the simulator to a previous visit to that cell. Very recently \cite{ecoffet2020first}, the same authors proposed a way to relax this requirement by training goal-conditioned policies to reliably return to cells in the archive during the exploration phase. In this paper, we investigated the theoretical dimension of this direction, by provably learning such goal-conditioned policies for the set of incrementally controllable states.


\textbf{Future work.} Interesting directions for future investigation include: \textbf{1)} Deriving a lower bound for the \DC problems; 
\textbf{2)} Integrating \ALGO into the meta-algorithm \MNM \citep{gajane2019autonomous} which deals with incremental exploration for \DCL in non-stationary environments; \textbf{3)} Extending the problem to continuous state space and function approximation; \textbf{4)} Relaxing the definition of incrementally controllable states and relaxing the performance definition towards allowing the agent to have a non-zero but limited sample complexity of learning a shortest-path policy for any state at test time.
\section*{Broader Impact}

This paper makes contributions to the fundamentals of online learning (RL) and due to its theoretical nature, we see no ethical or immediate societal consequence of our work.






\bibliographystyle{unsrt}
\bibliography{bibliography}

\newpage
\appendix
\part{Appendix}
%


\begin{figure}[b]
	\begin{minipage}[t]{0.4\linewidth}
		\begin{tikzpicture}[thick,scale=0.9,rotate=90,transform shape]
		\node[circle,draw=black,fill=white,rotate=-90,transform shape] at (1,2.5) (1) {\scriptsize{$s_0$}}; 
		\node[circle,fill=black] at (0,1.5) (2) {}; 
		\node[circle,fill=black] at (2,1.5) (4) {}; 
		\node[circle,fill=black] at (0,-1) (5) {}; 
		\node[circle,fill=black] at (0,0.5) (2bis) {}; 
		\node[circle,fill=black] at (2,0.5) (4bis) {}; 
		\node[circle,fill=black] at (2,-1) (7) {}; 
		\node[circle,fill=black] at (2,-2) (8) {}; 
		\node[circle,draw=black,fill=white,rotate=-90,transform shape] at (0,-2) (9) {\scriptsize{$x$}}; 
		\node[circle,draw=black,fill=white,rotate=-90,transform shape] at (2,-3) (10) {\scriptsize{$y$}}; 
		\begin{scope}[>={Stealth[black]},
		every node/.style={fill=white,circle},
		every edge/.style={draw=gray, thick},
		every loop/.style={draw=gray, thick, min distance=5mm,looseness=5}]
		\path[]
		(1) [->,thick] edge[] node[scale=0.001, text width = 0mm] {} (2)
		(1) [->,thick] edge[] node[scale=0.001, text width = 0mm] {} (4)
		(2) [->,thick] edge[] node[scale=0.001, text width = 0mm] {} (2bis)
		(2bis) [->,thick] edge[dashed] node[scale=0.001, text width = 0mm] {} (5)
		(4) [->,thick] edge[] node[scale=0.001, text width = 0mm] {} (4bis)
		(4bis) [->,thick] edge[dashed] node[scale=0.001, text width = 0mm] {} (7)
		(5) [->,thick] edge[] node[scale=0.001, text width = 0mm] {} (9)
		(7) [->,thick] edge[] node[scale=0.001, text width = 0mm] {} (8)
		(8)[->,thick] edge[] node[scale=0.001, text width = 0mm] {} (10)
		(10)[->,thick] edge[bend left=30] node[scale=0.001, text width = 0mm] {} (9);
		\end{scope}
		\end{tikzpicture}
	\end{minipage}%
	\begin{minipage}[t]{0.03\linewidth}
		\hfill
	\end{minipage}%
	\begin{minipage}[t]{0.57\linewidth}
		\vspace{-1in}
		\caption{\small{Let $\mathcal{X} := \{ s_0 \} \cup \{ x \}$ and $\mathcal{Y} := \mathcal{X} \cup \{ y \}$. For any $l \geq 1$, suppose that from $s_0$, the agent reaches $x$ in $l$ steps with probability $1/2$, or reaches $y$ in $l+1$ steps with probability $1/2$. If the goal state is $x$, constraining an agent to use policies restricted to $\mathcal{X}$ (i.e., that reset to $s_0$ outside of $\mathcal{X}$) is detrimental since $x$ can actually be reached in 1 step from $y$. Formally, we can easily prove that $V^{\star}_{\mathcal{X}}(s_0 \rightarrow x) - V^{\star}_{\mathcal{Y}}(s_0 \rightarrow x) = l + 1$, which grows arbitrarily as $l$ increases.}}
    \label{fig_toy}
	\end{minipage}
\end{figure}
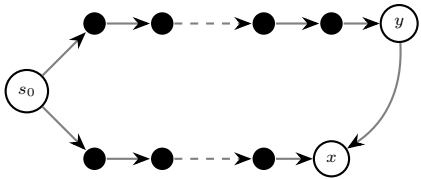

\section{Autonomous Exploration Objectives}
\label{app_objectives}

We recall the two \DC objectives stated in Def.\,\ref{definition_exploration_bound}: for any length $L \geq 1$, error threshold $\epsilon > 0$ and confidence level $\delta \in (0,1)$, the sample complexities $\mathcal{C}_{\DCL}(\mathfrak{A},L,\epsilon,\delta)$ and $\mathcal{C}_{\small\textsc{AX}^{\star}}(\mathfrak{A},L,\epsilon,\delta)$ are defined as the number of time steps required by a learning algorithm $\mathfrak{A}$ to identify a set $\mathcal{K} \supseteq \mathcal{S}_L^{\rightarrow}$ such that with probability at least $1 - \delta$, it has learned a set of policies $\{ \pi_s \}_{s \in \mathcal{K}}$ that respectively verifies 
the following \DC requirement
\begin{itemize}[itemsep=0.1pt,topsep=0pt]
    \item[\textup{(\DCL)}] ~ $ \forall s \in \mathcal{K}, v_{\pi_s}(s_0 \rightarrow s) \leq L + \epsilon$,
    \item[\textup{(\DCstar)}] ~ $\forall s \in \mathcal{K}, v_{\pi_s}(s_0 \rightarrow s) \leq V^{\star}_{\mathcal{S}_L^{\rightarrow}}(s_0 \rightarrow s) + \epsilon.$
\end{itemize}

As we explain in Sect.\,\ref{sect_analysis}, \ALGO (Alg.\,\ref{algorithm_SSP_generative_model}) succeeds in tackling condition \DCstar, whereas \UcbExplore~\cite{lim2012autonomous}, which is designed to tackle condition \DCL, is unable to tackle \DCstar. Note that the algorithmic design of \UcbExplore entails that it computes policies whose value function implicitly targets $V^{\star}_{\mathcal{K}_t}$, with $\mathcal{K}_t$ the \textit{current} set of controllable states. While $V^{\star}_{\mathcal{K}_t}$ is always smaller than $L$, \UcbExplore cannot provide any tightness guarantees w.r.t.\,$V^{\star}_{\mathcal{K}_t}$ since it has no guarantee that the transition dynamics are estimated well enough on $\mathcal{K}_t$. An additional challenge with which \UcbExplore fails to cope is the fact that the set $\mathcal{K}_t$ increases over time and thus unlocks new states and paths, which may be useful to improve its shortest-path policies for previously discovered states.

To better understand this phenomenon, let us introduce an alternative condition \DCprime --- tighter than \DCL, but looser than \DCstar --- which stems from the challenge of not knowing $\cS^{\rightarrow}_L$ in advance. We define \DCprime as follows: for any state $s$ in $\mathcal{S}_L^{\rightarrow}$, the objective is to find a policy that can reach $s$ from $s_0$ in at most $L'+\epsilon$ steps on average, where $L' := \min\{ l \leq L: s \in \mathcal{S}_l^{\rightarrow} \}$, i.e.,
\begin{itemize}[itemsep=0.1pt,topsep=0pt]
    \item[\textup{(\DC')}] ~ $\forall s \in \mathcal{K}, v_{\pi_s}(s_0 \rightarrow s) \leq L'+\epsilon $, where $L' := \min\{ l \leq L: s \in \mathcal{S}_l^{\rightarrow} \}$.
\end{itemize}

As mentioned in \cite[][Corollary~9]{lim2012autonomous}, it is possible to run separate instances of \UcbExplore with increasing $L_n = 1 + n \epsilon$ from $n = 0$ to $\lceil \frac{L-1}{\epsilon} \rceil$ (i.e., until $n$ satisfies $L_{n-1} \leq L \leq L_n$). This verifies the condition \DCprime at the cost of a worsened dependency on both $\epsilon$ and $L$ as follows
\begin{align*}
\mathcal{C}_{\DC'}(\UcbExplore,L,\epsilon,\delta) = \wt{O}\left(\frac{L^7 S_{\Lpluseps} A}{\epsilon^4} \right).
\end{align*}

While \DCprime is tighter than \DCL, it may be arbitrarily loose compared to \DCstar, which illustrates the intrinsic limitations in \UcbExplore design. \UcbExplore incrementally expands a set of \say{controllable} states $\mathcal{K}$: starting with $\mathcal{K}_0 = \{s_0\}$, at time $t$ a state $s$ is added to $\mathcal{K}_t$ whenever \UcbExplore can confidently assess that it managed to learn a policy reaching $s$ in less than $L$ steps. Since at time~$t$ \UcbExplore can only consider policies restricted to the controllable states $\mathcal{K}_t$, even the shortest-path policy computed to reach $s$ at time $t$ may not be $\epsilon$-optimal w.r.t.\ to the \textit{whole} set $\mathcal{S}_L^{\rightarrow}$. Indeed, every time a state is added to $\mathcal{K}$, this state may unlock new paths which may, for previously controllable states, allow for better shortest-path policies restricted on the updated  $\mathcal{K}$. Fig.\,\ref{fig_toy} illustrates this behavior, where the state $y$ unlocks a fast path from $y$ to $x$ which should be taken in $y$ instead of resetting to $s_0$. Consequently, if the agent seeks to tackle condition \DCstar, it must have the faculty to \textit{backtrack}, i.e., continuously update both its belief of the vicinity ($\mathcal{K}$) and its notion of optimality on the vicinity ($V^{\star}_{\mathcal{K}}$). Unfortunately, \UcbExplore can only compute policies targeting $V^{\star}_{\mathcal{K}}$ with $\mathcal{K}$ the \textit{current} set of controllable states, but it fails to be accurate enough to \textit{revise} such policies as the set of controllable states~$\mathcal{K}$ is expanded over time. In contrast, in virtue of its allocation function $\phi$ (Eq.\,\ref{allocation_function}) which enables to track the number of collected samples as $\mathcal{K}$ increases, \ALGO is able to improve its candidate shortest-path policies during the consolidation step \ding{176} when the \textit{final} set $\mathcal{K}$ is considered.

The following general and simple statement captures how the expansion of the state space of interest may alter and refine the optimal policy restricted on it.



\begin{lemma}\label{lemma_1}
	For any two sets $\mathcal{X} \subseteq \mathcal{Y}$ 
	and any state $x \in \mathcal{X}$, we have $V^{\star}_{\mathcal{X}}(s_0 \rightarrow x) \geq V^{\star}_{\mathcal{Y}}(s_0 \rightarrow x)$. Moreover, the gap between the two quantities may be arbitrarily large.
\end{lemma}

\begin{proof}
	The inequality is immediate from Asm.\,\ref{assumption_reset}. Fig.\,\ref{fig_toy} shows the gap may be arbitrarily large.
\end{proof}

%

Finally, we summarize all the sample complexity results in Tab.\,\ref{tab:comparison}.

\begin{table*}
	\begin{minipage}[t]{0.7\linewidth}
		\centering 
		{\renewcommand{\arraystretch}{2}
			\begin{tabular}{|c|c|c|} 
				\hline
				\textbf{\DC} & \textbf{\UcbExplore} \cite{lim2012autonomous} & \textbf{\ALGO} (Alg.\,\ref{algorithm_SSP_generative_model}) \\
				\hline
				\DCL & \small{$\wt{O}\left( \displaystyle\frac{L^6 S_{\Lpluseps} A}{\epsilon^3} \right)$} & \multirow{3}{*}{$\wt{O}\left( \displaystyle\frac{L^5 \Gamma_{\Lpluseps} S_{\Lpluseps} A}{\epsilon^2} + \frac{L^3 S_{\Lpluseps}^2 A}{\epsilon}\right)$} \\
				\cline{1-2}
				\DCprime & \small{$\wt{O}\left( \displaystyle\frac{L^7 S_{\Lpluseps} A}{\epsilon^4} \right)$} &  \\
				\cline{1-2}
				\DCstar & \textit{Unable} &  \\
				\hline
			\end{tabular}
		}
	\end{minipage}%
	\begin{minipage}[t]{0.3\linewidth}
		\vspace{-0.5in}
		\caption{Comparison between the sample complexity of \UcbExplore and \ALGO, depending on the condition \DCL, \DCprime or \DCstar.}
		\label{tab:comparison}
	\end{minipage}
\end{table*}

\section{Efficient Computation of Optimistic SSP Policy}
\label{app_blackbox_computation_optimistic_SSP_policy}

In this section we recall from \cite{tarbouriech2019no, cohen2020near} how to efficiently compute an optimistic stochastic shortest-path (SSP) policy.
%

\subsection{Computation of Optimal Policy in Known SSP}
\label{app_value_iteration_SSP}

This section details the procedure to efficiently compute an (arbitrarily near-) optimal policy $\pi$ in a \textit{known} SSP instance with positive costs and which admits at least one proper policy. Recall that a \textit{proper policy} is a policy whose execution starting from any non-goal state eventually reaches the goal state with probability one \cite{bertsekas1995dynamic}.

\begin{definition}[SSP-MDP]\label{def:ssp.mdp.app}
                An SSP-MDP is an MDP $M = (\mathcal{S}^\dagger, \mathcal{A}, s^\dagger, p, c)$ where $\mathcal{S}^\dagger$ is the set of non-goal states with $|\mathcal{S}^\dagger| = S^\dagger$, $\mathcal{A}$ is the set of actions, $p$ is the transition function and $c$ is the cost function. The goal state $s^\dagger \notin \mathcal{S}^\dagger$ is zero-cost and absorbing, i.e., $p(s^\dagger|s^\dagger,a) = 1$ and $c(s^\dagger,a) =0$ for any $a \in \mathcal{A}$.
\end{definition}

The (possibly unbounded) \textit{value function} (also called expected cost-to-go) of any policy $\pi \in \Pi$ starting from state $s_0$ is defined as
\begin{align*}
    V^{\pi}(s_{0}) := \mathbb{E}\bigg[ \sum_{t = 1}^{+\infty} c(s_{t}, \pi(s_t)) \,\Big\vert\,s_{0}\bigg] = \mathbb{E}\bigg[ \sum_{t = 1}^{\tau_{\pi}(s_0 \rightarrow s^{\dagger})} c(s_{t}, \pi(s_t)) \,\Big\vert\,s_{0}\bigg].
\end{align*}

\begin{assumption}\label{asm:ssp.mdp.app}
        We restrict the attention to SSP-MDP $M$ (see Def.~\ref{def:ssp.mdp.app}) such that, for any $(s,a) \in \mathcal{S}^\dagger \times \mathcal{A}$, $c(s,a) \in [c_{\min},1]$ with $c_{\min} > 0$. (Note that having positive costs ensures that for any non-proper policy $\pi$ there exists a state $s$ with $V^\pi(s) = +\infty$.) Moreover, we assume that there exists at least one proper policy (i.e., that reaches the goal state $s^{\dagger}$ with probability one starting from any state in $\mathcal{S}^\dagger$).
\end{assumption}

    The procedure \VISSP considers the following inputs: a goal $s^{\dagger}$, non-goal states $\cS^{\dagger}$, a known model $p$ and a known cost function $c$, with (non-goal) costs lower bounded by $c_{\min} > 0$. \VISSP outputs a vector $u$ (of size $\abs{\cS^{\dagger}}$) and a policy $\pi$ which is greedy w.r.t.\,the vector $u$.

The optimal Bellman operator is defined as follows for any vector $u$ and any non-goal state $s \in \cS^{\dagger}$
\begin{align*}
    \mathcal{L}u(s) := \min_{a \in \cA} \Big\{ c(s,a) + \sum_{s' \in  \cS^{\dagger}} p(s' \vert s,a) u(s') \Big\}.
\end{align*}
Note that by definition, $V^\pi(s^\dagger) = 0$ for any $\pi$.
We perform a value iteration (\VI) scheme over this operator as explained in~\cite[e.g.,][]{bertsekas2013stochastic,bonet2007speed,tarbouriech2019no}. Namely, we consider initial vector $u_0 := 0$ and set iteratively $u_{i+1} := \mathcal{L} u_{i}$ (see Alg.~\ref{alg:vi.ssp}). For a predefined \VI precision $\gamma > 0$, the stopping condition is reached for the first iteration $j$ such that $\norm{u_{j+1} - u_{j} }_{\infty} \leq \gamma$. The policy is then selected to be the greedy policy w.r.t.\,the vector $u := u_j$, i.e.,
\begin{align}
        \forall s \in \cS^\dagger \cup \{s^\dagger\}, \quad \pi(s) \in \argmin_{a \in \cA} \Big\{ c(s,a) + \sum_{s' \in \cS^\dagger} p(s' \vert s,a) u(s') \Big\}.
\label{eq_SSP_policy}
\end{align}
Importantly, while $u$ is \textit{not} the value function of $\pi$, both quantities can be related according to the following lemma.


\begin{algorithm}[t]
            \DontPrintSemicolon
            \caption{\VISSP}
            \label{alg:vi.ssp}
            \KwIn{Non-goal states $\cS^\dagger$, action set $\cA$, transitions $p$, costs $c$ and accuracy $\gamma$}
            \KwOut{Value vector $u$ and greedy policy $\pi$}
            Define $\mathcal{L}u(s) := \min_{a \in \cA} \Big\{ c(s,a) + \sum_{s' \in  \cS^{\dagger}} p(s' \vert s,a) u(s') \Big\}$\;
            Set $u_0 = \boldsymbol{0}_{S^\dagger}$ and $j=0$\;
            $u_1 = \mathcal{L}u_0$\;
            \While{$\|u_{j+1} - u_j\|_{\infty} > \gamma$}{
                    $u_{j+1} = \mathcal{L} u_j$
            }
            Set $u := u_j$ and $\pi(s)  \in \argmin_{a \in \cA} \Big\{ c(s,a) + \sum_{s' \in \cS^\dagger} p(s' \vert s,a) u(s') \Big\}$ for any $s \in \cS^\dagger \cup \{s^\dagger\}$
    \end{algorithm}

\begin{lemma}\label{lemma_app_value_iteration_SSP}
        Consider an SSP-MDP $M = (\mathcal{S}^\dagger, \cA, s^\dagger, p, c)$ defined as in Def.~\ref{def:ssp.mdp.app} and satisfying Asm.~\ref{asm:ssp.mdp.app}. Let $(u,\pi) = \text{\VISSP}(\mathcal{S}^\dagger, \cA, p, c, \gamma)$ be the solution computed by \VISSP.
Denote by $V^\pi$ the true value function of $\pi$ and by $V^\star = V^{\pi^\star} = \mathcal{L}V^\star$ the optimal value function.
The following component-wise inequalities hold
\begin{itemize}
    \item $u \leq V^{\star} \leq V^{\pi}$.
    \item If the \VI precision level verifies $\gamma \leq \frac{c_{\min}}{2}$, then $V^{\pi} \leq \left(1 + \frac{2 \gamma}{c_{\min}}\right) u$.
\end{itemize}
\end{lemma}


\begin{proof}
The result can be obtained by adapting \citep[][Lem.~4 \& App.\,E]{tarbouriech2019no}. For the first inequality, given that we consider the initial vector $u_0 = 0$, we know that $0 \leq V^{\star}$ with $V^{\star} = \mathcal{L} V^{\star}$ by definition. By monotonicity of the operator $\mathcal{L}$ \citep{puterman2014markov,bertsekas1995dynamic}, we obtain $u_j \leq V^{\star} \leq V^{\pi}$. As for the second inequality, we introduce the following Bellman operators of a deterministic policy $\pi$ for any vector $u$ and state $s$,
\begin{align*}
    \mathcal{L}^{\pi} u(s) &:= c(s, \pi(s)) + \sum_{s' \in \cS} p(s' \vert s, \pi(s)) u(s'), \\
    \mathcal{T}^{\pi}_{\gamma} u(s) &:= \underbrace{c(s, \pi(s)) - \gamma}_{>0} + \sum_{s' \in \cS} p(s' \vert s, \pi(s)) u(s').
\end{align*}
Note that the SSP problem defined by the operator $\mathcal{T}^{\pi}_{\gamma}$ satisfies Asm.\,\ref{asm:ssp.mdp.app} since i) it has positive costs due to the condition $\gamma \leq \frac{c_{\min}}{2}$ and ii) the fact that $M$ satisfies Asm.\,\ref{asm:ssp.mdp.app} guarantees the existence of at least one proper policy in the model $p$. We can write component-wise
\begin{align*}
    \mathcal{T}^{\pi}_{\gamma} u_j = \mathcal{L}^{\pi} u_j - \gamma \myeqa \mathcal{L} u_j - \gamma \myineeqb u_j,
\end{align*}
where (a) uses that $\pi$ is the greedy policy w.r.t.\,$u_j$ and (b) stems from the chosen stopping condition which yields $\mathcal{L}u_j \leq u_j + \gamma$. By monotonicity of the operator $\mathcal{T}^{\pi}_{\gamma}$, we have for all $m > 0$, $(\mathcal{T}_{\gamma}^{\pi})^m u_j \leq u_j$. The asymptotic convergence of the operator in an SSP problem satisfying Asm.\,\ref{asm:ssp.mdp.app} (see e.g.,~\cite[][Prop.\,2.2.1]{bertsekas1995dynamic}) guarantees that taking the limit $m \rightarrow + \infty$ yields $W_{\gamma}^{\pi} \leq u_j$, where $W_{\gamma}^{\pi}$ is defined as the value function of policy $\pi$ in the model $p$ with $\gamma$ subtracted to all the costs, i.e.,
\begin{align*}
    W_{\gamma}^{\pi}(s) := \mathbb{E}\left[ \sum_{t=1}^{\tau_{\pi}(s)} \left(c(s_t, \pi(s_t) - \gamma\right) \vert s_1 = s \right] = V^{\pi}(s) - \gamma \mathbb{E}\left[ \tau_{\pi}(s)\right],
\end{align*}
where $\tau_{\pi}(s)$ denotes the (random) hitting time of policy $\pi$ to reach the goal starting from state $s$. Moreover, we have $c_{\min} \mathbb{E}\left[ \tau_{\pi}(s)\right] \leq V^{\pi}(s) \leq c_{\max} \mathbb{E}\left[ \tau_{\pi}(s)\right]$. Putting everything together, we thus get $\left( 1 - \frac{\gamma}{c_{\min}} \right) V^{\pi} \leq u_j$. Since $\gamma \leq \frac{c_{\min}}{2}$, we ultimately obtain
\begin{align*}
    V^{\pi} \leq \frac{1}{1 - \frac{\gamma}{c_{\min}}} u_j \leq \left(1 + \frac{2 \gamma}{c_{\min}}\right) u_j,
\end{align*}
where the last inequality uses the fact that $\frac{1}{1-x} \leq 1+2x$ holds for any $0 \leq x \leq \frac{1}{2}$.
\end{proof}


\subsection{Computation of Optimistic Model in Unknown SSP}
\label{app_opt_model}

Consider an SSP problem $M$ defined as in Asm.~\ref{asm:ssp.mdp.app}.
Consider that, at any given stage of the learning process, the agent is equipped with $N(s,a)$ samples at each state-action pair. A method to compute an optimistic model $\wt{p}$ is provided in \cite{cohen2020near}, which we recall below.

Denote by $\wh{p}$ the current empirical average of transitions: $\wh{p}(s'|s,a) = N(s,a,s') / N(s,a)$, and set $\widehat{\sigma}^2(s'\vert s,a) := \wh{p}(s'|s,a) (1 - \wh{p}(s'|s,a))$ as well as $N^{+}(s,a) := \max \{1,N(s,a) \} $. For any $(s,a,s') \in \cS^{\dagger} \times \mathcal{A} \times \cS^{\dagger}$, the empirical Bernstein inequality \cite{audibert2007tuning, maurer2009empirical} is leveraged to select the following confidence intervals (with probability at least $1-\delta$) on the transition probabilities
\begin{align*}
    \beta(s,a,s') := 2 \sqrt{ \frac{\widehat{\sigma}^2(s'\vert s,a)}{N^{+}(s,a)} \log\left(\frac{2 S A N^{+}(s,a)}{\delta} \right)} + \frac{6 \log\left(\frac{2 S A N^{+}(s,a)}{\delta} \right)}{N^{+}(s,a)},
\end{align*}
and $\beta(s,a,s^{\dagger}) := \sum_{s' \in \cS^{\dagger}} \beta(s,a,s')$.
The selection of the optimistic model $\wt{p}$ is as follows: the probability of reaching the goal $s^{\dagger}$ is maximized at every state-action pair, which implies minimizing the probability of reaching all other states and setting them at the lowest value of their confidence range. Formally, we set for all $(s,a,s') \in \cS^{\dagger} \times \mathcal{A} \times \cS^{\dagger}$,
\begin{align*}
    \wt{p}(s' \vert s,a) := \max\Big\{ \wh{p}(s' \vert s,a) - \beta(s,a,s') , ~ 0 \Big\},
\end{align*}
and $\wt{p}(s^{\dagger} \vert s,a) := 1 - \sum_{s' \in \cS^{\dagger}} \wt{p}(s' \vert s,a)$.

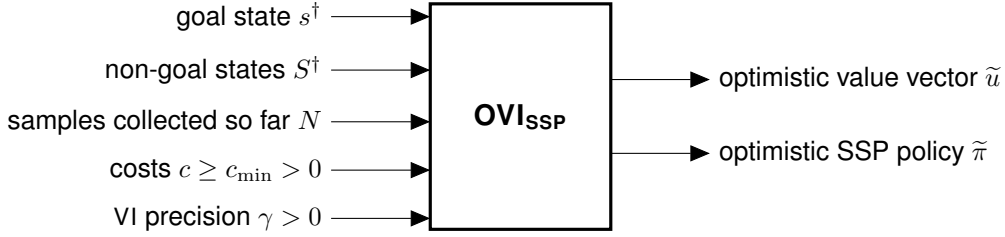
\begin{figure}[t]
        \centering
        \begin{tikzpicture}[scale=1.2,every node/.style={transform shape},font=\sffamily,>=triangle 45]
        \node [circuit] (item) {\OVISSP};
        \matrix[
            matrix of nodes,
            left= of item,
            row sep=\myheight/25,
            nodes={anchor=east}
            ] (rightmatr) {
            goal state $s^{\dagger}$\\
            non-goal states $S^{\dagger}$\\
            samples collected so far $N$\\
            costs $c \geq c_{\min} > 0$\\
            \VI precision $\gamma > 0$\\
        };
        \matrix[
            matrix of nodes,
            right= of item,
            row sep=\myheight/6,
            nodes={anchor=west}
            ] (leftmatr) {
            optimistic value vector $\wt u$ \\
            optimistic SSP policy $\wt \pi$ \\
        };
        \foreach \i in {1,2,3,4,5}
            \draw [->] (rightmatr-\i-1) -- (rightmatr-\i-1 -| item.west);
        \foreach \i in {1,2}
            \draw [<-] (leftmatr-\i-1) -- (leftmatr-\i-1 -| item.east);
        \end{tikzpicture}
        \caption{Optimistic Value Iteration for SSP (\OVISSP).}
        \label{fig_OVI_SSP}
    \end{figure}

\subsection{Combining the two: Optimistic Value Iteration for SSP (\OVISSP)}

\OVISSP first computes an optimistic model $\wt{p}$ leveraging App.\,\ref{app_opt_model}, and it then runs the \VISSP procedure of App.\,\ref{app_value_iteration_SSP} in the model $\wt p$, i.e., $(\wt u, \wt \pi) = \text{\VISSP}(\mathcal{S}^\dagger, \mathcal{A}, s^\dagger, \wt p, c)$. This outputs an optimistic pair $(\wt u, \wt \pi)$ composed of the \VI vector $\wt{u}$ and the policy $\wt{\pi}$ that is greedy w.r.t.\,$\wt{u}$ in the model $\wt p$. The \OVISSP scheme is recapped in Fig.\,\ref{fig_OVI_SSP}.

\section{Useful Result: Simulation Lemma for SSP}
\label{app_sample_complexity_SSP_gen_model}

Consider a stochastic shortest-path (SSP) instance (see Def.\,\ref{def:ssp.mdp.app}) that satisfies Asm.\,\ref{asm:ssp.mdp.app}. We denote by $A = \abs{\cA}$ the number of actions, $S = \abs{\cS}$ the number of non-goal states, $g \notin \cS$ the (zero-cost and absorbing) goal state, $p$ the unknown transitions and $c$ the known cost function. We assume that $0 < c(s,a) \leq 1$ for all $(s,a) \in \cS \times \cA$, and set $c_{\min} := \min_{s,a} c(s,a) > 0$. We also set $\cS' := \cS \cup \{g\}$. Recall that the goal state is zero-cost (i.e., $c(g,a)=0$) and absorbing (i.e., $p(g \vert g,a)=1$), and that the value function of a policy amounts to the expected cumulative costs following this policy until reaching the goal.

\begin{definition}
For any model $p$ and $\eta > 0$, we introduce the set of models close to $p$ w.r.t.\,the $\ell_1$-norm on the non-goal states as follows
\begin{align*}
  \mathcal{P}_{\eta}^{(p)} := \Big\{ p' \in \mathbb{R}^{S' \times A \times S'} : \quad &\forall (s,a) \in \SA, ~p'(\cdot \vert s,a) \in \Delta(\cS'), ~ p(g \vert g,a) = 1, \\ & \sum_{y \in \cS} \abs{p(y \vert s,a) - p'(y \vert s,a)} \leq \eta \Big\}.
\end{align*}

\end{definition}

\begin{lemma}[Simulation Lemma for SSP]\label{lemma_simulation_ssp}
    Consider any model $p$ and $p' \in \mathcal{P}_\eta^{(p)}$ such that, for each model, there exists at least one proper policy w.r.t.\,the goal state $g$. Consider any policy $\pi$ that is proper in $p'$, with value function denoted by $V_{\pi}'$, such that the following condition is verified
    \begin{align}
        \eta \norm{V_{\pi}'}_{\infty} \leq 2 c_{\min}.
    \label{eq_key_sim_lemma_ssp}
    \end{align}
    Then $\pi$ is proper in $p$ (i.e., its value function verifies $V_{\pi} < + \infty$ component-wise), and we have
    \begin{align*}
        \forall s \neq g, ~ V_{\pi}(s ) \leq \left( 1 +  \frac{2\eta \norm{V'_{\pi}}_{\infty}}{c_{\min}} \right) V'_{\pi}(s ),
    \end{align*}
    and conversely,
    \begin{align*}
    \forall s \neq g, ~ V'_{\pi}(s ) \leq \left(1 + \frac{\eta \norm{V'_{\pi}}_{\infty}}{c_{\min}} \right) V_{\pi}(s ).
    \end{align*}
    Combining the two inequalities above yields
    \begin{align*}
        \norm{ V_{\pi} - V'_{\pi}}_{\infty} \leq \frac{7 \eta \norm{V'_{\pi}}_{\infty}^2}{c_{\min}}.
    \end{align*}
\end{lemma}

\begin{proof}

The proof of Lem.\,\ref{lemma_simulation_ssp} requires a result of \citep{bertsekas1991analysis} recalled in Lem.\,\ref{lemma_technical_lemma_ssp} and can be seen as a generalization of \citep[][Lem.\,B.4]{cohen2020near}.
First, let us assume that $\pi$ is proper in the model $p'$. This implies that its value function, denoted by $V'$, is bounded component-wise. Moreover, for any non-goal state $s \in \cS$, the Bellman equation holds as follows
\begin{align}
    V'(s) &= c(s, \pi(s)) + \sum_{y \in \cS} p'(y \vert s, \pi(s)) V'(y) \nonumber \\
    &= c(s, \pi(s)) + \sum_{y \in \cS} p(y \vert s, \pi(s)) V'(y) + \sum_{y \in \cS} \left( p'(y \vert s, \pi(s)) - p(y \vert s, \pi(s)) \right) V'(y).
\label{eq_sim_lemma}
\end{align}
By successively using H\"older's inequality and the facts that $p' \in \mathcal{P}_{\eta}^{(p)}$ and $c(s, \pi(s)) \geq c_{\min}$, we get
\begin{align*}
    V'(s) \geq c(s, \pi(s)) - \eta \norm{V'}_{\infty} + p(\cdot \vert s, \pi(s)) ^\top V' \geq c(s, \pi(s)) \left(1 - \frac{\eta \norm{V'}_{\infty}}{c_{\min}}\right) + p(\cdot \vert s, \pi(s)) ^\top V'.
\end{align*}
Let us now introduce the vector $V'' := \left(1 - \frac{\eta \norm{V'}_{\infty}}{c_{\min}}\right)^{-1} V'$. Then for all $s \in \cS$,
\begin{align*}
    V''(s) \geq c(s,\pi(s)) + p(\cdot \vert s,\pi(s)) ^\top V''.
\end{align*}
Hence, from Lem.\,\ref{lemma_technical_lemma_ssp}, $\pi$ is proper in $p$ (i.e., $V < + \infty$), and we have
\begin{align}
    V \leq V'' \leq \left(1 + 2 \frac{\eta \norm{V'}_{\infty}}{c_{\min}} \right) V',
    \label{eq_sim_lemma_1}
\end{align}
where the last inequality stems from condition \eqref{eq_key_sim_lemma_ssp} and the fact that $\frac{1}{1-x} \leq 1+2x$ holds for any $0 \leq x \leq \frac{1}{2}$. Conversely, analyzing Eq.\,\ref{eq_sim_lemma} from the other side, we get
\begin{align*}
    V'(s) \leq c(s, \pi(s)) \left( 1 + \frac{\eta \norm{V'}_{\infty}}{c_{\min}} \right) + p(\cdot \vert s, \pi(s)) ^\top V'.
\end{align*}
Let us now introduce the vector $V'' := \left(1 + \frac{\eta \norm{V'}_{\infty}}{c_{\min}}\right)^{-1} V'$. Then
\begin{align*}
    V''(s) \leq c(s,\pi(s)) + p(\cdot \vert s,\pi(s)) ^\top V''.
\end{align*}
We then obtain in the same vein as Lem.\,\ref{lemma_technical_lemma_ssp} (by leveraging the monotonicity of the Bellman operator $\mathcal{L}^{\pi}U(s) := c(s,\pi(s)) + p(\cdot \vert s,\pi(s)) ^\top U$) that $V'' \leq V$, and therefore
\begin{align}
    V' \leq \left(1 + \frac{\eta \norm{V'}_{\infty}}{c_{\min}} \right) V.
    \label{eq_sim_lemma_2}
\end{align}
Combining Eq.\,\ref{eq_sim_lemma_1} and \ref{eq_sim_lemma_2} yields component-wise
\begin{align*}
    \norm{V - V'}_{\infty} \leq 2 \frac{\eta \norm{V'}_{\infty}}{c_{\min}} \norm{V'}_{\infty} + \frac{\eta \norm{V'}_{\infty}}{c_{\min}} \norm{V}_{\infty} \leq 7 \frac{\eta \norm{V'}_{\infty}^2}{c_{\min}},
\end{align*}
where the last inequality uses that $\norm{V}_{\infty} \leq 5 \norm{V'}_{\infty}$ which stems from plugging condition \eqref{eq_key_sim_lemma_ssp} into Eq.\,\ref{eq_sim_lemma_1}.

Note that here $p$ and $p'$ play symmetric roles; we can perform the same reasoning in the case where $\pi$ is proper in the model $p$ and it would yield an equivalent result by switching the dependencies on $V$ and $V'$.
\end{proof}

\begin{lemma}[\cite{bertsekas1991analysis}, Lem.\,1]
  In an SSP-MDP satisfying Asm.\,\ref{asm:ssp.mdp.app}, let $\pi$ be any policy, then
    \begin{itemize}[leftmargin=.3in,topsep=-4pt,itemsep=0pt,partopsep=0pt, parsep=0pt]
        \item If there exists a vector $U: \cS \rightarrow \mathbb{R}$ such that $U(s) \geq c(s, \pi(s)) + \sum_{s' \in \cS} p(s' \vert s, \pi(s)) U(s')$ for all $s \in \cS$, then $\pi$ is proper, and $V^{\pi}$ the value function of $\pi$ is upper bounded by $U$ component-wise, i.e., $V^{\pi}(s ) \leq U(s)$ for all $s \in \cS$.
        \item If $\pi$ is proper, then its value function $V^{\pi}$ is the unique solution to the Bellman equations $V^{\pi}(s ) = c(s, \pi(s)) + \sum_{s' \in \cS} p(s' \vert s, \pi(s)) V^{\pi}(s' )$ for all $s \in \cS$.
    \end{itemize}
\label{lemma_technical_lemma_ssp}
\end{lemma}

\section{Proof of Theorem \ref{theorem_bound_UCSSPGM} (Sample Complexity Analysis of \ALGOtitle)}
\label{app_full_proof}

\subsection{Computation of the Optimistic Policies} \label{app:app_full_proof.ovi}
At each round $k$, for each goal state $s^\dagger \in \mathcal{W}_k$,  \ALGO computes an optimistic goal-oriented policy associated to the MDP $M'_k(s^\dagger)$ constructed as in Def.~\ref{definition_induced_mdp}.
This MDP is defined over the entire state space $\cS$ and restricts the action to the only action \RESET outside $\cK_k$.
We can build an equivalent MDP by restricting the focus on $\cK_k$. To this end, we define the following SSP-MDP.

\begin{definition}\label{def:ssp.mdp.onKk}
Define $M^\dagger_k(s^\dagger) := \langle \cS^\dagger_k, \cA^\dagger_k(\cdot), c^\dagger_k, p^\dagger_k \rangle$ where $\cS^\dagger_k := \mathcal{K}_k \cup \{ s^\dagger , x \}$ and $S_k^\dagger = |\cS^\dagger_k| = |\cK_k| + 2$. State $x$ is a meta-state that encapsulates all the states that have been observed so far and are not in $\cK_k$.
The action space $\cA^\dagger_k(\cdot)$ is such that $\cA^\dagger_k(s) = \cA$ for all states $s \in \mathcal{K}_k$ and $\cA^\dagger_k(s) = \{ \RESET \}$ for $s \in \{s^\dagger, x\}$.
The cost function is $c^\dagger_k(x,a) = 0$ for any $a \in \cA^\dagger_k(x)$ and $c^\dagger_k(s,a) =1$ everywhere else.
The transition function is defined as $p^\dagger_k(s^\dagger|s^\dagger,a) = p_k^\dagger(s_0|x,a) = 1$ for any  $a$,
$p^\dagger_k(y|s,a) = p(y|s,a)$ for any $(s,a,y) \in \cK_k \times  \cA \times (\cK_k \cup \{ s^{\dagger} \})$ and $p^\dagger_k(x|s,a) = 1 - \sum_{y \in \cK_k \cup \{ s^{\dagger} \}} p^\dagger_k(y|s,a)$.
\end{definition}

Note that solving $M^{\dagger}_k$ yields a policy effectively restricted to the set $\cK_k$ insofar as we can interpret the meta-state $x$ as $\cS \setminus \{\cK_k \cup \{s^\dagger\}\}$. Since $p$ is unknown, we cannot construct $M_k^\dagger(s^\dagger)$. Let $N_k$ be the state-action counts accumulated up until now. We denote by $\wh{p}_k$ the \say{global} empirical estimates, i.e., $\wh{p}_k(y|s,a) = N_k(s,a,y) / N_k(s,a)$. Given them, we define the \say{restricted} empirical estimates $\wh{p}^\dagger_k$ as follows: $\wh{p}^\dagger_k(y|s,a) := \wh{p}_k(y|s,a)$ for any $(s,a,y) \in \cK_k \times  \cA \times (\cK_k \cup \{ s^{\dagger} \})$ and $\wh{p}^\dagger_k(x|s,a) := 1 - \sum_{y \in \cK_k \cup \{ s^{\dagger} \}} \wh{p}^\dagger_k(y|s,a)$. Denoting $N_k^{+}(s,a) := \max \{1,N_k(s,a) \}$, we then define the following bonuses for any $(s,a,y) \in \cK_k \times  \cA \times (\cK_k \cup \{ s^{\dagger} \})$, 
\begin{align}
    \beta_{k}(s,a,y) &:= 2 \sqrt{ \frac{\wh{p}_k(y|s,a)(1 - \wh{p}_k(y|s,a))}{N^{+}_{k}(s,a)} \log\left(\frac{2 S A N^{+}_{k}(s,a)}{\delta} \right)} + \frac{6 \log\left(\frac{2 S A N^{+}_{k}(s,a)}{\delta} \right)}{N^{+}_{k}(s,a)}, \\
    \beta_{k}(s,a,x) &:= \sum_{y \in \cK_k \cup \{ s^{\dagger} \}} \beta_{k}(s,a,y).
    \label{bonuses}
\end{align}
Moreover, we set the uncertainty about the MDP at the meta-state $x$ and at the goal state $s^{\dagger}$ to $0$ by construction (since their outgoing transitions are deterministic, respectively to $s_0$ and $s^{\dagger}$).

We now leverage the optimistic construction mentioned in App.\,\ref{app_value_iteration_SSP}. 

\begin{definition} \label{def:ssp.mdp.onKk.optimistic}
        We denote by $\wt{M}_k^\dagger(s^\dagger) = \langle \cS^\dagger_k, \cA^\dagger_k(\cdot), c_k^\dagger, \wt p^\dagger_k \rangle$ the optimistic MDP associated to $M^\dagger_k(s^\dagger)$ defined in Def.\,\ref{def:ssp.mdp.onKk}.
Then, $\forall (s,a) \in \mathcal{K}_{k} \times \mathcal{A}$,
\begin{align}
    \wt{p}^\dagger_k(y \vert s,a) &:= \max\left\{ \wh{p}_k(y \vert s,a) - \beta_{k}(s,a,y) , ~ 0 \right\}, \quad \forall y \in \mathcal{K}_{k} \cup \{ x\}, \\
    \wt{p}_k^{\dagger}(s^\dagger \vert s,a) &:= 1 - \sum_{y \in \mathcal{K}_{k} \cup \{ x\}} \wt{p}^\dagger_k(y \vert s,a), \\
    \wt{p}^\dagger_k(s^\dagger|s^\dagger,a) &= \wt p^\dagger_k(s_0|x,a) = 1.
\end{align}
\end{definition}

\begin{algorithm}[t]
        \DontPrintSemicolon
        \caption{\OVISSP}\label{alg:OVISSP.app}
        \KwIn{$\cK_k$, $\cA$, $s^\dagger$, $N_k$, $\gamma > 0$}
        \KwOut{Value vector $\wt{u}^\dagger$ and policy $\wt \pi^\dagger$}
        \BlankLine
        Estimate transitions probabilities $\wh{p}_k$ using $N_k$\;
        Compute the optimistic SSP-MDP $\wt{M}^\dagger_k$ as detailed in Def.~\ref{def:ssp.mdp.onKk.optimistic}\;
        Compute $(\wt u^\dagger_k, \wt \pi^\dagger_k) = \text{\VISSP}(\cS^\dagger_k, \cA^\dagger_k, c_k^\dagger, \wt p^\dagger_k , \gamma)$ (see Alg.~\ref{alg:vi.ssp})
\end{algorithm}

Given this MDP, we can compute the optimistic value vector $\wt u^\dagger_k$ and policy $\wt \pi^\dagger_k$ using value iteration for SSP: $(\wt u^\dagger_k, \wt \pi^\dagger_k) = \text{\VISSP}(\cS^\dagger_k, \cA^\dagger_k, c_k^\dagger, \wt p^\dagger_k , \frac{\epsilon}{4L})$.
We summarize the construction of the optimistic model and the computation of value function and policy in Alg.\,\ref{alg:OVISSP.app} (\OVISSP).

\paragraph{Remark.} \jt{Note that the structure of the problem does not appear to allow for variance-aware improvements in the analysis of Thm.\,\ref{theorem_bound_UCSSPGM} (specifically, when the analysis will apply an SSP simulation lemma argument). Indeed, given the possibly large number of states in the total environment $\mathcal{S}$, the computation of the optimistic policies requires the construction of the meta-state $x$ that encapsulates all the states in $\cS \setminus \{\cK_k \cup \{s^\dagger\}\}$, where $s^{\dagger}$ is the candidate goal state considered at round $k$. As a result, the uncertainty on the transitions reaching $x$ needs to be summed over multiple states, as shown in Eq.\,\ref{bonuses}. This extra uncertainty at a single state in the induced MDP has the effect of canceling out Bernstein techniques seeking to lower the prescribed requirement of the state-action samples that the algorithm should collect. In turn this implies that such variance-aware techniques would not lead to any improvement in the final sample complexity bound.}

\subsection{High-Probability Event}

\begin{lemma}
\label{lemma_high_prob}
    It holds with probability at least $1-\delta$ that for any time step $t \geq 1$ and for any state-action pair $(s,a)$ and next state $s'$,
\begin{align}
\abs{\widehat{p}_{t}(s' \vert s,a) - p(s' \vert s,a)} \leq 2 \sqrt{ \frac{\widehat{\sigma}^2_{t}(s'\vert s,a)}{N^{+}_{t}(s,a)} \log\left(\frac{2 S A N^{+}_{t}(s,a)}{\delta} \right)} + \frac{6 \log\left(\frac{2 S A N^{+}_{t}(s,a)}{\delta} \right)}{N^{+}_{t}(s,a)},
\label{empirical_b_ineq}
\end{align}
where $N^{+}_{t}(s,a) := \max \{ 1, N_t(s,a) \}$ and where $\wh{\sigma}_t^2$ are the population variance of transitions, i.e., $\wh{\sigma}_t^2(s' \vert s,a) := \wh{p}_t(s' \vert s,a)(1-\wh{p}_t(s' \vert s,a))$.
\end{lemma}

\begin{proof}
The confidence intervals in Eq.\,\ref{empirical_b_ineq} are constructed using the empirical Bernstein inequality, which guarantees that the considered event holds with probability at least $1 - \delta$, see e.g., \cite{improved_analysis_UCRL2B}.
\end{proof}

Define the set of plausible transition probabilities as
\begin{align*}
    C_k^{\dagger} := \bigcap_{(s,a) \in \mathcal{S}_k^{\dagger} \times \cA } C_k^{\dagger}(s,a),
\end{align*}
where
\begin{align*}
 C_k^{\dagger}(s,a) := \{ \widetilde{p} \in \mathcal{C} ~\vert ~ \widetilde{p}(\cdot \, \vert \, s^{\dagger},a) = \mathds{1}_{s^{\dagger}}, \widetilde{p}(\cdot \, \vert \, x,a) = \mathds{1}_{s_0}, \abs{\widetilde{p}(s' \vert s,a) - \widehat{p}_k(s' \vert s,a)} \leq \beta_k(s,a,s')\},
\end{align*}
with $\mathcal{C}$ the $S^\dagger_k$-dimensional simplex and $\widehat{p}_k$ the empirical average of transitions.


\begin{lemma}\label{lemma_high_probability_intersection_bound_M_in_M_k}
Introduce the event $\Theta := \bigcap_{k=1}^{+ \infty} \bigcap_{s^{\dagger} \in \mathcal{W}_k} \{ p_k^{\dagger} \in C_k^{\dagger} \}$. Then $\mathbb{P}(\Theta) \geq 1 - \frac{ \delta}{3}$.
\end{lemma}

\begin{proof}
We have with probability at least $1-\frac{\delta}{3}$ that, for any $y \neq x$, $\abs{p^\dagger_k(y|s,a) - \wh{p}^\dagger_k(y|s,a) } \leq \beta_{k}(s,a,y)$ from the empirical Bernstein inequality (see Eq.\,\ref{empirical_b_ineq}), and moreover $\abs{\wh{p}^\dagger_k(x|s,a) - p^\dagger_k(x|s,a)} = \left\vert  1 - \sum_{y \in \cK_k \cup \{ s^{\dagger} \}} p^\dagger_k(y|s,a) - \left( 1 - \sum_{y \in \cK_k \cup \{ s^{\dagger} \}} \wh{p}^\dagger_k(y|s,a) \right)  \right\vert \leq \sum_{y \in \cK_k \cup \{ s^{\dagger} \}} \abs{p^\dagger_k(y|s,a) - \wh{p}^\dagger_k(y|s,a) } \leq \beta_{k}(s,a,x)$.
\end{proof}

\begin{lemma}\label{lemma_ptilde_in_P_eta}
  Under the event $\Theta$, for any round $k$ and any goal state $s^{\dagger} \in \mathcal{W}_k$, the optimistic model $\wt{p}_k^{\dagger}$ constructed in Def.\,\ref{def:ssp.mdp.onKk.optimistic} verifies $\wt{p}_k^{\dagger} \in \mathcal{P}^{(p_k^{\dagger})}_{\eta_k}$, with $\eta_k := 4 \beta_{k}(s,a,x)$ where $\beta_k$ is defined in Eq.\,\ref{bonuses}.
\end{lemma}

\begin{proof} Combining the construction in Def.\,\ref{def:ssp.mdp.onKk.optimistic}, the proof of Lem.\,\ref{lemma_high_probability_intersection_bound_M_in_M_k} and the triangle inequality yields
\begin{align*}
    \sum_{y \in \cK_k \cup \{ x \}} \abs{\wt{p}_k^{\dagger}(y \vert s,a) - p_k^{\dagger}(y \vert s,a)} &\leq  \sum_{y \in \cK_k \cup \{ x \}} \abs{\wt{p}_k^{\dagger}(y \vert s,a) - \wh{p}_k^{\dagger}(y \vert s,a)} + \abs{\wh{p}_k^{\dagger}(y \vert s,a) - p_k^{\dagger}(y \vert s,a)} \\
    &\leq \sum_{y \in \cK_k \cup \{ x \}} \beta_{k}(s,a,y)  + 2 \beta_{k}(s,a,x) \\
    &\leq 4 \beta_{k}(s,a,x).
\end{align*}
\end{proof}
Throughout the remainder of the proof, we assume that the event $\Theta$ holds.


\subsection{Properties of the Optimistic Policies and Value Vectors}

We recall notation. Let us fix any round $k$ and any goal state $s^{\dagger} \in \mathcal{W}_k$. We denote by $\wt{\pi}_k^{\dagger}$ the greedy policy w.r.t.~$\wt{u}^{\dagger}_k(\cdot \rightarrow s^{\dagger})$ in the optimistic model $\wt{p}^{\dagger}_k$. Let $\wt{v}^{\dagger}_k(s \rightarrow s^{\dagger})$ be the value function of policy $\wt{\pi}^{\dagger}_k$ starting from state $s$ in the model $\wt{p}^{\dagger}_k$.
We can apply Lem.\,\ref{lemma_app_value_iteration_SSP} given that the conditions of Asm.\,\ref{asm:ssp.mdp.app} hold (indeed, we have $c_{\min} = 1 > 0$ and there exists at least one proper policy to reach the goal state $s^{\dagger}$ since it belongs to $\mathcal{W}_k$). Moreover, we have that $\wt{V}_{\mathcal{K}_k}^{\star}(s_0 \rightarrow s^{\dagger}) \leq V_{\mathcal{K}_k}^{\star}(s_0 \rightarrow s^{\dagger})$ given the way the optimistic model $\wt{p}^{\dagger}_k$ is computed (i.e., by maximizing the probability of transitioning to the goal at any state-action pair), see \citep[][Lem.\,B.12]{cohen2020near}. Hence we get the two following important properties.

\begin{lemma}\label{lem:optimism}
    For any round $k$, goal state $s^{\dagger} \in \mathcal{W}_k$ and state $s \in \mathcal{K}_k \cup \{ x \}$, we have under the event~$\Theta$,
    \begin{align*}
        \widetilde{u}^{\dagger}_k(s \rightarrow s^{\dagger}) \leq V_{\mathcal{K}_k}^{\star}(s \rightarrow s^{\dagger}).
    \end{align*}
\end{lemma}

\begin{lemma}\label{lemma_relation_u_v}

For any round $k$, goal state $s^{\dagger} \in \mathcal{W}_k$ and state $s \in \mathcal{K}_k \cup \{ x \}$, we have 
\begin{align*}
    \wt{v}^{\dagger}_k(s \rightarrow s^{\dagger}) \leq (1+ 2 \gamma) \wt{u}^{\dagger}_k(s \rightarrow s^{\dagger}).
\end{align*}
\end{lemma}


\subsection{State Transfer from $\mathcal{U}$ to $\mathcal{K}$ (step \ding{175})}
\label{app_state_transfer}


We fix any round $k$ and any goal state $s^{\dagger} \in \mathcal{W}_k$ that is added to the set of \say{controllable} states $\mathcal{K}$, i.e., for which $\wt{u}_k^{\dagger}(s_0 \rightarrow s^{\dagger}) \leq L$.

\begin{lemma}
Under the event $\Theta$, we have both following inequalities
\begin{align*}
        \begin{cases}
        v_k^{\dagger}(s_0 \rightarrow s^{\dagger}) \leq
            L + \epsilon, \\
        v_k^{\dagger}(s_0 \rightarrow s^{\dagger}) \leq
            V_{\mathcal{K}_k}^{\star}(s_0 \rightarrow s^{\dagger}) + \epsilon.
    \end{cases}
\end{align*}
In particular, the first inequality entails that $s^{\dagger} \in \mathcal{S}_{L+\epsilon}^{\rightarrow}$, which justifies the validity of the state transfer from $\mathcal{U}$ to $\mathcal{K}$.
\label{lemma_states_discovered}
\end{lemma}

\begin{proof}
We have
\begin{align}
    \wt{v}_k^{\dagger}(s_0 \rightarrow s^{\dagger}) \myineeqa (1+2 \gamma) \wt{u}_k^{\dagger}(s_0 \rightarrow s^{\dagger}) \leq \left\{
    \begin{array}{ll}
        \myineeqb L + \frac{\epsilon}{3} \\
      \myineeqc V_{\mathcal{K}_k}^{\star}(s_0 \rightarrow s^{\dagger}) + \frac{\epsilon}{3},
    \end{array}
\right.
\label{ineq_}
\end{align}
where inequality (a) comes from Lem.\,\ref{lemma_relation_u_v}, inequality (b) combines the algorithmic condition $\wt{u}_k^{\dagger}(s_0 \rightarrow s^{\dagger}) \leq L$ and the \VI precision level $\gamma := \frac{\epsilon}{6 L}$, and finally inequality (c) combines Lem.\,\ref{lem:optimism} and the \VI precision level. Moreover, for any state in $\mathcal{K}_k$,
\begin{align*}
    \wt{v}_k^{\dagger}(s \rightarrow s^{\dagger}) \myineeqa \wt{V}_{\mathcal{K}_k}^{\star}(s \rightarrow s^{\dagger}) + \frac{\epsilon}{3} \myineeqb \wt{V}_{\mathcal{K}_k}^{\star}(s_0 \rightarrow s^{\dagger}) + 1 + \frac{\epsilon}{3} \leq \wt{v}_k^{\dagger}(s_0 \rightarrow s^{\dagger}) + 1 + \frac{\epsilon}{3},
\end{align*}
where (a) comes from Lem.\,\ref{lem:optimism} and (b) stems from the presence of the \RESET action (Asm.\,\ref{assumption_reset}).

We now provide the exact choice of allocation function $\phi$ in Alg.\,\ref{algorithm_SSP_generative_model}. We introduce $$\gamma := \frac{2 \epsilon}{12 (L + 1 + \epsilon)(L + \frac{\epsilon}{3})}.$$ (Note that $\gamma = O(\epsilon/ L^2)$.) We set the following requirement of samples for each state-action pair $(s,a)$ at round $k$,
\begin{align}\label{budget_modest.detail}
    n_k = \phi(\mathcal{K}_k) =  \left\lceil \frac{57 X_k^2}{\gamma^2} \left[ \log\left(\frac{8 e X_k \sqrt{2 S A} }{\sqrt{\delta}\gamma}   \right) \right]^2 + \frac{24 \abs{\mathcal{S}_k^{\dagger}}}{\gamma} \log\left( \frac{24 \abs{\mathcal{S}_k^{\dagger}} S A}{\delta \gamma} \right) \right\rceil,
\end{align}
where we define $$X_k := \max_{(s,a) \in \mathcal{S}_k^{\dagger} \times \cA } \sum_{s' \in \mathcal{S}_k^{\dagger} } \sqrt{\widehat{\sigma}^2_{k}(s'\vert s,a)},$$
with $\widehat{\sigma}^2_{k}(s'\vert s,a) := \wh p_k^{\dagger}(s'\vert s,a) (1-\wh p_k^{\dagger}(s'\vert s,a))$ the estimated variance of the transition from $(s,a)$ to $s'$. Leveraging the empirical Bernstein inequality (Lem.\,\ref{lemma_high_prob}) and perfoming simple algebraic manipulations (see e.g., \cite[][Lem.\,8 and 9]{kazerouni2017conservative}) yields that $\beta_k(s,a,x) \leq \gamma$. From Lem.\,\ref{lemma_ptilde_in_P_eta}, this implies that $\wt{p}_k^{\dagger} \in \mathcal{P}^{(p_k^{\dagger})}_{\eta}$ with $\eta := 4 \gamma$.
We can then apply Lem.\,\ref{lemma_simulation_ssp} (whose condition~\ref{eq_key_sim_lemma_ssp} is verified), which gives
\begin{align}
    v_k^{\dagger}(s_0 \rightarrow s^{\dagger}) &\leq \left(1 + \eta \norm{\wt{v}_k^{\dagger}(\cdot \rightarrow s^{\dagger})}_{\infty} \right) \wt{v}_k^{\dagger}(s_0 \rightarrow s^{\dagger}) \label{eq_application_sim_lemma} \\
    &\leq \left(1 + \eta (L + 1 + \epsilon) \right) \wt{v}_k^{\dagger}(s_0 \rightarrow s^{\dagger}) \nonumber \\
    &\leq  \wt{v}_k^{\dagger}(s_0 \rightarrow s^{\dagger}) + \frac{2\epsilon}{3}, \nonumber
\end{align}
where the last inequality uses that $\eta (L + 1 + \epsilon)(L + \frac{\epsilon}{3}) = \frac{2 \epsilon}{3}$ by definition of $\gamma$. Plugging in Eq.\,\ref{ineq_} yields the sought-after inequalities.

\end{proof}


\subsection{Termination of the Algorithm}
\label{app_stopping_condition_algorithm}

\begin{lemma}[Variant of Lem.\,17 of \cite{lim2012autonomous}] Suppose that for every state $s \in \mathcal{S}$, each action $a \in \mathcal{A}$ is executed $b \geq \lceil L \log\left( \frac{3 A L S}{\delta}\right) \rceil$ times. Let $\mathcal{S}'_{s,a}$ be the set of all next states visited during the $b$ executions of $(s,a)$. Denote by $\Lambda$ the complementary of the event
\begin{align*}
\left\{ \exists (s',s,a) \in \mathcal{S}^2 \times \mathcal{A} : p(s' \vert s,a) \geq \frac{1}{L} \wedge s' \notin \mathcal{S}'_{s,a} \right\}.
\end{align*}
Then $\mathbb{P}( \Lambda) \geq 1 - \frac{\delta}{3}$.
\label{lemma_high_prob_visit}
\end{lemma}

\begin{lemma} Under the event $\Theta \cap \Lambda$, for any round $k$, either $\mathcal{S}_L^{\rightarrow} \subseteq \mathcal{K}_k$, or there exists a state $s^{\dagger} \in \mathcal{S}_L^{\rightarrow} \setminus \mathcal{K}_k$ such that $s^{\dagger} \in \mathcal{W}_k$ and is $L$-controllable with a policy restricted to $\mathcal{K}_k$. Moreover, $\abs{\mathcal{W}_k} \leq 2 L A \abs{\mathcal{K}_k}$.
\label{lem:W_k}
\end{lemma}

\begin{proof}[Proof of Lem.\,\ref{lem:W_k}]
Consider a round $k$ such that $\mathcal{S}_L^{\rightarrow} \setminus \mathcal{K}_k$ is non-empty.
Due to the incremental construction of the set $\mathcal{S}_L^{\rightarrow}$ (Def.~\ref{definition_incremental_set}), there exists a state $s^{\dagger} \in \mathcal{S}^{\rightarrow}_L$ and a policy restricted to $\mathcal{K}_k$ that can reach $s^{\dagger}$ in at most $L$ steps (in expectation). Hence there exists a state-action pair $(s,a) \in \mathcal{K}_k \times \mathcal{A}$ such that $p(s^{\dagger} \vert s,a) \geq \frac{1}{L}$. Since $\phi(\mathcal{K}_k) \geq \lceil L \log\left( \frac{3 A L S}{\delta}\right) \rceil$ samples are available at each state-action pair, according to Lem.\,\ref{lemma_high_prob_visit}, we get that, under the event $\Lambda$, $s^{\dagger}$ is found during the sample collection procedure for the state-action pair $(s,a)$ (step \ding{172}), which implies that $s^{\dagger} \in \mathcal{U}_k$.

Moreover, the choice of allocation function $\phi$ guarantees in particular that there are more than $\Omega(\frac{4 L^2}{\epsilon^2} \log( \frac{2 L S A}{\delta \epsilon} ))$ samples available at each state-action pair $(s,a) \in \mathcal{K}_k \times \cA$. From the empirical Bernstein inequality of Eq.\,\ref{empirical_b_ineq}, we thus have that $\abs{p(s^{\dagger} \vert s,a) - \wh{p}_k(s^{\dagger} \vert s,a)} \leq \frac{\epsilon}{2 L}$ under the event $\Theta$. 
Consequently we have
\begin{align*}
    \wh{p}_k(s^{\dagger} \vert s,a) \geq \frac{1}{L} - \abs{p(s^{\dagger} \vert s,a) - \wh{p}_k(s^{\dagger} \vert s,a)} \geq \frac{1 - \frac{\epsilon}{2 }}{L},
\end{align*}
which implies that $s^{\dagger} \in \mathcal{W}_k$. Furthermore, we can decompose $\mathcal{W}_k$ the following way
\begin{align*}
    \mathcal{W}_k = \bigcup_{(s,a) \in \mathcal{K}_k \times \mathcal{A}} \mathcal{Y}_k(s,a),
\end{align*}
where we introduce the subset
\begin{align*}
    \mathcal{Y}_k(s,a) := \left\{ s' \in \mathcal{U}_k: \wh{p}_k(s' \vert s,a) \geq \frac{1 - \frac{\epsilon}{2 }}{L} \right\}.
\end{align*}
We then have
\begin{align*}
    1 = \sum_{s' \in \mathcal{S}} \wh{p}_k(s' \vert s,a) \geq \sum_{s' \in \mathcal{Y}_k(s,a)} \wh{p}_k(s' \vert s,a) \geq\frac{1 - \frac{\epsilon}{2 }}{L} \abs{\mathcal{Y}_k(s,a)}.
\end{align*}
We conclude the proof by writing that
\begin{align*}
    \abs{\mathcal{W}_k} \leq \sum_{(s,a) \in \mathcal{K}_k \times \mathcal{A}} \abs{\mathcal{Y}_k(s,a)} \leq \frac{L}{1 - \frac{\epsilon}{2 }} A \abs{\mathcal{K}_k} \leq 2 L A \abs{\mathcal{K}_k},
\end{align*}
where the last inequality uses that $\epsilon \leq 1$ (from line \ref{line_epsilon_min} of Alg.\,\ref{algorithm_SSP_generative_model}).
\end{proof}


\begin{lemma}
Under the event $\Theta \cap \Lambda$, when either condition \StopOne or \StopTwo is triggered (at a round indexed by $K$), we have $\mathcal{S}_L^{\rightarrow} \subseteq \mathcal{K}_K$.
\label{lemma_ending_2}
\end{lemma}
\begin{proof}
If condition \StopOne is triggered, Lem.\,\ref{lem:W_k} immediately guarantees that $\mathcal{S}_L^{\rightarrow} \subseteq \mathcal{K}_K$ under the event $\Lambda$. If condition \StopTwo is triggered, we have for all $s \in \mathcal{W}_K$, $\wt{u}_{s}(s_0 \rightarrow s) > L$. From Lem.\,\ref{lem:optimism} this means that, under the event $\Theta$, for all $s \in \mathcal{W}_K$, $V^{\star}_{\mathcal{K}_K}(s_0 \rightarrow s) > L$. Hence none of the states in $\mathcal{W}_K$ can be reached in at most $L$ steps (in expectation) with a policy restricted to $\mathcal{K}_K$. We conclude the proof using Lem.\,\ref{lem:W_k}.
\end{proof}


\begin{lemma}
Under the event $\Theta \cap \Lambda$, when \ALGO terminates at round $K$, for any state $s \in \mathcal{K}_K$, the policy $\pi_s$ computed during step \ding{176} verifies
\begin{align*}
    v_{\pi_{s}}(s_0 \rightarrow s) \leq \min_{\pi \in \Pi(\mathcal{S}_L^{\rightarrow})} v_{\pi}(s_0 \rightarrow s) + \epsilon.
\end{align*}
Moreover, we have that $\mathcal{S}_{L}^{\rightarrow} \subseteq \mathcal{K}_K \subseteq \mathcal{S}_{L+\epsilon}^{\rightarrow}$.
\label{lemma_final_policies_shortest_path}
\end{lemma}

\begin{proof}
Assume that the event $\Theta \cap \Lambda$ holds. Then when the final set $\mathcal{K}_K$ is considered and the new policies are computed using all the samples, Lem.\,\ref{lemma_states_discovered} yields for all $s \in \mathcal{K}_K$,
\begin{align*}
    v_{\pi_{s}}(s_0 \rightarrow s) \leq
\min_{\pi \in \Pi(\mathcal{K}_K)} v_{\pi}(s_0 \rightarrow s) + \epsilon.
\end{align*}
Moreover Lem.\,\ref{lemma_ending_2} entails that $\mathcal{K}_K \supseteq \mathcal{S}_L^{\rightarrow}$. This implies from Lem.\,\ref{lemma_1} that
\begin{align*}
\min_{\pi \in \Pi(\mathcal{K}_K)} v_{\pi}(s_0 \rightarrow s)
&\leq \min_{\pi \in \Pi(\mathcal{S}_L^{\rightarrow})} v_{\pi}(s_0 \rightarrow s),
\end{align*}
which means that $\mathcal{K}_K \subseteq \mathcal{S}_{L+\epsilon}^{\rightarrow}$.
\end{proof}







\subsection{High Probability Bound on the Sample Collection Phase (step \ding{172})}

Denote by $K$ the (random) index of the last round during which the algorithm terminates. We focus on the sample collection procedure for any state $s \in \mathcal{K}_K$. We denote by $k_s$ the index of the round during which $s$ was added to the set of \say{controllable} states $\mathcal{K}$. To collect samples at state $s$, the learner uses the shortest-path policy $\pi_s$. We say that an attempt to collect a specific sample is a \textit{rollout}. We denote by $Z_K := \abs{\mathcal{K}_K} A N_K$ the total number of samples that the learner needs to collect. As such, at most $Z_K$ rollouts must take place. Assume that the event $\Theta$ holds. Then from Lem.\,\ref{lemma_final_policies_shortest_path}, we have $\mathcal{K}_K \subseteq \mathcal{S}_{L+\epsilon}^{\rightarrow}$. Hence, denoting $S_{L + \epsilon} := \abs{\mathcal{S}_{L+\epsilon}^{\rightarrow}}$, we have $Z_K \leq Z_{L + \epsilon} := S_{L + \epsilon} A \Phi(\mathcal{S}_{L+\epsilon}^{\rightarrow})$.
The following lemma provides a high-probability upper bound on the time steps required to meet the sampling requirements.

\begin{lemma}
Assume that the event $\Theta$ holds. Set
\begin{align*}
    \psi := 4 (L + \epsilon + 1) \log\left(\frac{6 Z_{L+\epsilon}}{\delta}\right),
\end{align*}
and introduce the following event
\begin{align*}
    \mathcal{T} := \Big\{ &\exists \textrm{~one rollout (with goal state $s$)}~ \textrm{s.t.~} \tau_{\pi_s}(s_0 \rightarrow s) > \psi \Big\}.
\end{align*}
We have $\mathbb{P}\left(\mathcal{T}\right) \leq \frac{\delta}{3}$.
\label{lem:gen_model}
\end{lemma}

\begin{proof}
Assume that the event $\Theta$ holds. Leveraging a union bound argument and applying Lem.\,\ref{lemma_quantify_goal_reaching_probability} to policy $\pi_s$ which verifies $v_{\pi_s}(s' \rightarrow s) \leq L + \epsilon + 1$ for any $s' \in K_{k_s}$, we get
\begin{align*}
    \mathbb{P}\left(\mathcal{T}\right) \leq \sum_{rollouts} 2 \exp\left( - \frac{\psi}{4 (L + \epsilon + 1)} \right) \leq 2 Z_{L+\epsilon} \exp\left( - \frac{\psi}{4 (L + \epsilon + 1)} \right) \leq \frac{\delta}{3},
\end{align*}
where the last inequality comes from the choice of $\psi$.
\end{proof}


\begin{lemma}[\cite{cohen2020near}, Lem.\,B.5]\label{lemma_quantify_goal_reaching_probability}
Let $\pi$ be a proper policy such that for some $d > 0$, $V_{\pi}(s) \leq d$ for every non-goal state $s$. Then the probability that the cumulative cost of $\pi$ to reach the goal state from any state $s$ is more than $m$, is at most $2 e^{-m/(4 d)}$ for all $m \geq 0$. Note that a cost of at most $m$ implies that the number of steps is at most $m / c_{\min}$.
\end{lemma}

\subsection{Putting Everything Together: Sample Complexity Bound}
\label{app_putting_everything_together}

The sample complexity of the algorithm is solely induced by the sample collection procedure (step~\ding{172}). Recall that we denote by $K$ the index of the round at which the algorithm terminates. With probability at least $1-\frac{2\delta}{3}$, Lem.\,\ref{lemma_ending_2} holds, and so does the event $\Theta$. Hence the algorithm discovers a set of states $\mathcal{K}_K \supseteq \mathcal{S}_L^{\rightarrow}$. Moreover, from Lem.\,\ref{lemma_final_policies_shortest_path}, the algorithm outputs for each $s \in \mathcal{K}_K$ a policy $\pi_s$ with $\mathbb{E}\left[ \tau_{\pi_s}(s_0 \rightarrow s) \right] \leq V^{\star}_{\mathcal{S}_L^{\rightarrow}}(s) + \epsilon$. Hence we also have $\abs{\mathcal{K}_K} \leq S_{L + \epsilon} := \abs{\mathcal{S}_{L+\epsilon}^{\rightarrow}}$.

We denote by $Z_K := \abs{\mathcal{K}_K} A \, \phi(\mathcal{K}_K)$ the total number of samples that the learner needs to collect. From Lem.\,\ref{lem:gen_model}, with probability at least $1-\frac{\delta}{3}$, the total sample complexity of the algorithm is at most $\psi Z_K$, where $\psi := 4 (L + \epsilon + 1) \log\left(\frac{6 Z_{L+\epsilon}}{\delta}\right)$.

Now, from Eq.\,\ref{budget_modest.detail} there exists an absolute constant $\alpha > 0$ such that \ALGO selects as allocation function $\phi$
\begin{align*}
    \phi : \mathcal{X} \rightarrow \alpha \cdot \left( \frac{L^4 \wh{\Theta}(\mathcal{X})}{\epsilon^2} \log^2 \left( \frac{L S A }{\epsilon \delta} \right) + \frac{L^2 \abs{\mathcal{X}}}{ \epsilon} \log\left( \frac{L S A }{\epsilon \delta} \right) \right),
\end{align*}
where
\begin{align*}
  \wh{\Theta}(\mathcal{X}) := \max_{(s,a) \in \mathcal{X} \times \mathcal{A}} \left( \sum_{s' \in \mathcal{X}} \sqrt{ \wh{p}(s' \vert s,a)(1 - \wh{p}(s' \vert s,a)) } \right)^2.
\end{align*}
The total requirement is $\phi(\mathcal{K}_K)$. Note that from Cauchy-Schwarz's inequality, we have
\begin{align*}
    \wh{\Theta}(\mathcal{K}_K) \leq \Gamma_{K} := \max_{(s,a) \in \mathcal{K}_K \times \mathcal{A}} \norm{\{ p(s' \vert s, a)\}_{s' \in \mathcal{K}_K}}_0 \leq \abs{\mathcal{K}_K}.
\end{align*}
Combining everything yields with probability at least $1-\delta$,
\begin{align*}
    \psi Z_K = \wt{O}\left( \frac{L^5 \Gamma_{K} \abs{\mathcal{K}_K} A}{\epsilon^2} + \frac{L^3 \abs{\mathcal{K}_K}^2 A}{\epsilon}\right).
\end{align*}
We finally use that $\mathcal{K}_K \subset \mathcal{S}_{L+\epsilon}^{\rightarrow}$ from Lem.\,\ref{lemma_final_policies_shortest_path}, which implies that
\begin{align*}
        \mathcal{C}_{\small\textsc{AX}^{\star}}(\ALGO, L, \epsilon, \delta) = \wt{O}\left( \frac{L^5 \Gamma_{L+\epsilon} S_{L+\epsilon} A}{\epsilon^2} + \frac{L^3 S_{L+\epsilon}^2 A}{\epsilon}\right),
\end{align*}
where $\Gamma_{L + \epsilon} := \max_{(s,a) \in \mathcal{S}_{L+\epsilon}^{\rightarrow} \times \mathcal{A}} \norm{\{ p(s' \vert s, a)\}_{s' \in \mathcal{S}_{L+\epsilon}^{\rightarrow}}}_0$. This concludes the proof of Thm.\,\ref{theorem_bound_UCSSPGM}.

\subsection{Proof of Corollary~\ref{cor:cost.dependent}}
\label{app_corollary}

The result given in Cor.\,\ref{cor:cost.dependent} comes from retracing the analysis of Lem.\,\ref{lemma_final_policies_shortest_path} and therefore Lem.\,\ref{lemma_states_discovered} by considering non-uniform costs between $[c_{\min}, 1]$ instead of costs all equal to $1$. Specifically, Eq.\,\ref{eq_application_sim_lemma} needs to account for the inverse dependency on $c_{\min}$ of the simulation lemma of Lem.\,\ref{lemma_simulation_ssp}. This induces the final $\epsilon / c_{\min}$ accuracy level achieved by the policies output by \ALGO. There remains to guarantee that condition \ref{eq_key_sim_lemma_ssp} of Lem.\,\ref{lemma_simulation_ssp} is verified. In particular the condition holds if $\eta (L + 1 + \epsilon) \leq 2 c_{\min}$, where~$\eta$ is the model accuracy prescribed in the proof of Lem.\,\ref{lemma_states_discovered}. We see that this is the case whenever we have $\epsilon = O(L c_{\min})$ due to the fact that $\eta = \Omega(\epsilon/ L^2)$.

\subsection{Computational Complexity of \ALGO}
\label{subsection_computational_complexities}

The overall computational complexity of \ALGO can be expressed as $\sum_{k=1}^K \vert \mathcal{W}_k \vert \cdot C(\OVISSPmath)$, where $C(\OVISSPmath)$ denotes the complexity of an \OVISSP procedure and where we recall that $K$ denotes the (random) index of the last round during which the algorithm terminates. Note that it holds with high probability that $K \leq \vert S_{L+\epsilon}^{\rightarrow} \vert$ and $\vert \mathcal{W}_k \vert \leq 2 L A \vert \mathcal{K}_k \vert \leq 2 L A  \vert S_{L+\epsilon}^{\rightarrow} \vert$. Moreover $C(\OVISSPmath)$ captures the complexity of the value iteration (VI) algorithm for SSP, which was proved in \cite{bonet2007speed} to converge in time quadratic w.r.t.\,the size of the considered state space (here, $\mathcal{K}_k$) and $\norm{V^{\star}}_{\infty} / c_{\min}$. Here we have $c_{\min} = 1$, and we can easily prove that in all the SSP instances considered by \ALGO, the optimal value function $V^{\star}$ verifies $\norm{V^{\star}}_{\infty} = O(L^2)$, due to the restriction of the goal state in $\mathcal{W}_k$ (indeed this restriction implies that there exists a state-action pair in $\mathcal{K}_k \times \mathcal{A}$ that transitions to the goal state with probability $\Omega(1 / L)$ in the true MDP). Putting everything together gives \ALGO's computational complexity. Interestingly, we notice that while it depends polynomially on $S_{L+\epsilon}$, $L$ and $A$, it is independent from $S$ the size of the global state space.





\section{The \UcbExplorelarge Algorithm \citep{lim2012autonomous}}
\label{app_algo_ucbexplore}

\subsection{Outline of the Algorithm}
\label{app_ucbexplore}

The \UcbExplore algorithm was introduced by Lim and Auer~\cite{lim2012autonomous} to specifically tackle condition \DCL. The algorithm maintains a set $\mathcal{K}$ of \say{controllable} states and a set $\mathcal{U}$ of \say{uncontrollable} states. It alternates between two phases of \textit{state discovery} and \textit{policy evaluation}. In a state discovery phase, new candidate states are discovered as potential members of the set of controllable states. Any policy evaluation phase is called a \textit{round} and it relies on an optimistic principle: it attempts to reach an \say{optimistic} state $s$ (i.e., the easiest state to reach based on information collected so far) among all the candidate states by executing an optimistic policy $\pi_s$ that minimizes the optimistic expected hitting time truncated at a horizon of $H_{\Ucb} := \lceil L + L^2 \epsilon^{-1}\rceil$. Within the round of evaluation of policy $\pi_s$, the algorithm proceeds through at most $\lambda_{\Ucb} := \left\lceil 6L^3\epsilon^{-3} \log\left( 16 \abs{\mathcal{K}}^2 \delta^{-1}\right) \right\rceil$ episodes, each of which begins at $s_0$ and ends either when $\pi_s$ successfully reaches $s$ or when $H_{\Ucb}$ steps have been executed. If the \textit{empirical performance} of $\pi_s$ is poor (measured through a performance check done after each episode), the round is said to have \textit{failed}. Otherwise, the round is \textit{successful} which means that $s$ is controllable and an acceptable policy ($\pi_s$) has been discovered. A failure round leads to selecting another candidate state-policy pair for evaluation, while a success round leads to a state discovery phase which in turn adds more candidate states for the subsequent rounds. As explained in App.\,\ref{app_objectives}, \UcbExplore is unable to tackle the more challenging objective \DCstar.

\subsection{Minor Issue and Fix in the Analysis of \UcbExplore}
\label{subsection_mistake_fix}

The key insight of \UcbExplore is to bound the number of \textit{failure rounds} of the algorithm, by lower- and upper-bounding the so-called \say{regret} contribution of failure rounds, where the regret of a failure round $k$ is defined as
\begin{align*}
    \sum_{j=1}^{e_k} \Big[ H_{\Ucb} - L - \sum_{i=0}^{\Gamma - 1} r_i \Big],
\end{align*}
where $e_k \leq \lambda_{\Ucb}$ is the actual number of episodes executed in round $k$ and where the reward $r_i \in \{0, 1 \}$ is equal to 1 only if the state is the goal state. However, upper bounding the regret contribution of failure rounds implies applying a concentration inequality on \textit{only} specific rounds that are chosen given their \textit{empirical performance}. Hence Lim and Auer~\cite[][Lem.\,18]{lim2012autonomous} improperly use a martingale argument to bound a sum whose summands are chosen in a non-martingale way, i.e., depending on their realization.

To avoid the aforementioned issue, one must upper and lower bound the cumulative regret of the \textit{entire} set of rounds and not \textit{only} the failure rounds in order to obtain a bound on the number of failure rounds. However, this would yield a sample complexity that has a second term scaling as $\wt{O}(\epsilon^{-4})$. Following personal communication with the authors, the fix is to change the definition of regret of a round, making it equal to
\begin{align*}
    \sum_{j=1}^{e_k} \wt{u}_{H_{\Ucb}}(s_0 \rightarrow s) - \sum_{i=0}^{H_{\Ucb} - 1} r_i,
\end{align*}
where $s$ is the considered goal state and $\wt{u}_{H_{\Ucb}}(s_0 \rightarrow s)$ is the optimistic $H_{\Ucb}$-step reward (where the reward is equal to 1 only at state $s$). With this new definition, it is possible to recover the sample complexity provided in~\cite{lim2012autonomous} scaling as $\wt{O}(\epsilon^{-3})$.

\subsection{Issue with a Possibly Infinite State Space}
\label{subsection_issue_infinite_state_space}

Lim and Auer~\cite{lim2012autonomous} claim that their setting can cope with a countable, possibly infinite state space. However, this leads to a technical issue, which has been acknowledged by the authors via personal communication and as of now has not been resolved. Indeed, it occurs when a union bound over the unknown set $\mathcal{U}$ is taken to guarantee high-probability statements (e.g., the Lem.\,14 or 17 of \cite{lim2012autonomous}). Yet for each realization of the algorithm, we do not know what the set $\mathcal{U}$, or equivalently $\mathcal{K}$, looks like, hence it is improper to perform a union bound over a set of unknown identity. Simple workarounds to circumvent this issue are to impose a finite state space, or to assume prior knowledge over a finite superset of $\mathcal{U}$. In this paper we opt for the first option. It remains an open and highly non-trivial question as to how (and whether) the framework can cope with an infinite state space.


\subsection{Effective Horizon of the \DCtitle Problem and its Dependency on $\epsilon$}
\label{app_effective_horizon_problem}

\UcbExplore~\cite{lim2012autonomous} designs finite-horizon problems with horizon $H_{\Ucb} := \lceil L + L^2 \epsilon^{-1}\rceil$ and outputs policies that reset every $H_{\Ucb}$ time steps. In the following we prove that the effective horizon of the \DC problem actually scales as $O\left(\log(L \epsilon^{-1}) L\right)$, i.e., only \textit{logarithmically} w.r.t.\,$\epsilon^{-1}$. We begin by defining the concept of \say{resetting} policies as follows.

\begin{definition}
    For any $\pi \in \Pi$ and horizon $H \geq 0$, we denote by $\pi^{\vert H}$ the non-stationary policy that executes the actions prescribed by $\pi$ and performs the \RESET action every $H$ steps, i.e.,
    \begin{align*}
        \pi^{\vert H}_t(a \vert s) := \left\{
    \begin{array}{ll}
        \RESET & \mbox{if } t \equiv 0\ (\textrm{mod}\ H), \\
        \pi(a \vert s) & \mbox{otherwise.}
    \end{array}
    \right.
    \end{align*}
We denote by $\Pi^{\vert H}$ the set of such \say{resetting} policies.
\end{definition}

The following lemma captures the effective horizon $H_{\textrm{eff}}$ of the problem, in the sense that restricting our attention to $\Pi^{\vert H}(\mathcal{S}_L^{\rightarrow})$ for $H \geq H_{\textrm{eff}}$ does not compromise the possibility of finding policies that achieve the performance required by \DCstar (and thus also by \DCL).


\begin{lemma}\label{lemma_log_dependency_H}
    For any $\epsilon \in (0,1]$ and $L \geq 1$, whenever
    \begin{align*}
        H \geq H_{\textup{\textrm{eff}}} := 4 (L+1) \big\lceil \log\big(\frac{4(L+1)}{\epsilon}\big) \big\rceil,
    \end{align*}
    we have for any $s^{\dagger} \in \mathcal{S}_L^{\rightarrow}$,
    \begin{align*}
        \min_{\pi^{\vert H} \in \Pi^{\vert H}(\mathcal{S}_L^{\rightarrow})} v_{\pi^{\vert H}}(s_0 \rightarrow s^{\dagger}) \leq V^{\star}_{\mathcal{S}_L^{\rightarrow}}(s_0 \rightarrow s^{\dagger}) + \epsilon.
    \end{align*}
\end{lemma}

\begin{proof}
Consider any goal state $s^{\dagger} \in \mathcal{S}_L^{\rightarrow}$. Set $\epsilon' := \frac{\epsilon}{2(L+1)} \leq \frac{1}{2}$. Denote by $\pi \in \Pi(\mathcal{S}_L^{\rightarrow})$ the minimizer of $V^{\star}_{\mathcal{S}_L^{\rightarrow}}(s_0 \rightarrow s^{\dagger})$. For any horizon $H \geq 0$, we introduce the truncated value function $v_{\pi,H}(s \rightarrow s') := \mathbb{E}\left[\tau_{\pi}(s \rightarrow s') \land H \right]$ and the tail probability $q_{\pi,H}(s \rightarrow s') := \mathbb{P}(\tau_{\pi}(s \rightarrow s') > H)$. Due to the presence of the \RESET action, the value function of $\pi$ can be bounded for all states $s \in \mathcal{S}_L^{\rightarrow} \setminus \{ s^{\dagger} \}$ as
\begin{align*}
    v_{\pi}(s \rightarrow s^{\dagger}) \leq V^{\star}_{\mathcal{S}_L^{\rightarrow}}(s_0 \rightarrow s^{\dagger}) + 1 \leq L + 1.
\end{align*}

This entails that the probability of the goal-reaching time decays exponentially. More specifically, we have
\begin{align}
    q_{\pi, H}(s_0 \rightarrow s^{\dagger}) \leq 2 \exp\left( -\frac{H}{4 (L+1)} \right) \leq \epsilon',
\label{eq_eff_horizon_1}
\end{align}
where the first inequality stems from Lem.\,\ref{lemma_quantify_goal_reaching_probability} and the second inequality comes from the choice of $H \geq 4(L+1) \big\lceil \log\big(\frac{2}{\epsilon'}\big) \big\rceil$. Furthermore, we have $\tau_{\pi}(s \rightarrow s') \land H \leq \tau_{\pi}(s \rightarrow s')$ and thus $\mathbb{E}\left[\tau_{\pi}(s \rightarrow s') \land H \right] \leq \mathbb{E}\left[\tau_{\pi}(s \rightarrow s')\right]$. Consequently,
\begin{align}
    v_{\pi,H}(s_0 \rightarrow s^{\dagger}) \leq v_{\pi}(s_0 \rightarrow s^{\dagger}) = V^{\star}_{\mathcal{S}_L^{\rightarrow}}(s_0 \rightarrow s^{\dagger}).
\label{eq_eff_horizon_2}
\end{align}
Now, from \cite[][Eq.\,4]{lim2012autonomous}, the value function of $\pi$ can be related to its truncated value function and tail probability as follows
\begin{align}
        v_{\pi^{\vert H}} = \frac{v_{\pi,H} + q_{\pi,H}  }{ 1 - q_{\pi,H} }.
    \label{proposition_equation_intuition}
\end{align}
Plugging Eq.\,\ref{eq_eff_horizon_1} and \ref{eq_eff_horizon_2} into Eq.\,\ref{proposition_equation_intuition} yields
\begin{align*}
    v_{\pi^{\vert H}}(s_0 \rightarrow s^{\dagger}) \leq \frac{V^{\star}_{\mathcal{S}_L^{\rightarrow}}(s_0 \rightarrow s^{\dagger}) + \epsilon'}{1 - \epsilon'}.
\end{align*}
Notice that the inequalities $\frac{1}{1-x} \leq 1+2x$ and $\frac{x}{1-x} \leq 2x$ hold for any $0 < x \leq \frac{1}{2}$. Applying them for $x = \epsilon'$ yields
\begin{align*}
    \frac{V^{\star}_{\mathcal{S}_L^{\rightarrow}}(s_0 \rightarrow s^{\dagger}) + \epsilon'}{1 - \epsilon'} \leq (1 + 2 \epsilon') V^{\star}_{\mathcal{S}_L^{\rightarrow}}(s_0 \rightarrow s^{\dagger}) + 2 \epsilon'.
\end{align*}
From the inequality $V^{\star}_{\mathcal{S}_L^{\rightarrow}}(s_0 \rightarrow s^{\dagger}) \leq L$ and the definition of $\epsilon'$, we finally obtain
\begin{align*}
    v_{\pi^{\vert H}}(s_0 \rightarrow s^{\dagger}) \leq V^{\star}_{\mathcal{S}_L^{\rightarrow}}(s_0 \rightarrow s^{\dagger}) + \epsilon,
\end{align*}
which completes the proof.
\end{proof}

Lem.\,\ref{lemma_log_dependency_H} reveals that the effective horizon $H_{\textrm{eff}}$ of the \DC problem scales only logarithmically and not linearly in~$\epsilon^{-1}$. This highlights that the design choice in \UcbExplore to tackle finite-horizon problems with horizon $H_{\Ucb}$ unavoidably leads to a suboptimal dependency on $\epsilon$ in its \DCL sample complexity bound. In contrast, by designing SSP problems and thus leveraging the intrinsic goal-oriented nature of the problem, \ALGO can (implicitly) capture the effective horizon of the problem. This observation is at the heart of the improvement in the $\epsilon$ dependency from $\wt{O}( \epsilon^{-3})$ of \UcbExplore~\cite{lim2012autonomous} to $\wt{O}( \epsilon^{-2})$ of \ALGO (Thm.\,\ref{theorem_bound_UCSSPGM}).



\section{Experiments} \label{app:experiments}

This section complements the experimental findings partially reported in Sect.\,\ref{sec:experiment}. We provide details about the algorithmic configurations and the environments as well as additional experiments.

\subsection{Algorithmic Configurations}
\label{app_alg_configurations}

\paragraph{Experimental improvements to \UcbExplore \cite{lim2012autonomous}.}
We introduce several modifications to \UcbExplore in order to boost its practical performance. We remove all the constants and logarithmic terms from the requirement for state discovery and policy evaluation (refer to~\citep[][Fig.\,1]{lim2012autonomous}).
Furthermore, we remove the constants in the definition of the accuracy $\epsilon' = \epsilon / L$ used by \UcbExplore (while their original algorithm requires $\epsilon'$ to be divided by $8$, we remove this constant). We also significantly improve the planning phase of \UcbExplore~\cite[][Fig.\,2]{lim2012autonomous}. Their procedure requires to divide the samples into $H := (1+ 1/\epsilon')L$ disjoint sets to estimate the transition probability of each stage $h$ of the finite-horizon MDP. This substantially reduces the accuracy of the estimated transition probability since for each stage $h$ only $N_k(s,a) / H$ are used. In our experiments, we use all the samples to estimate a stationary MDP (i.e., $\wh{p}_k(s'|s,a) = N_k(s,a,s')/N_k(s,a)$) rather than a stage-dependent model. 
Estimating a stationary model instead of bucketing the data is simpler and more efficient since leads to a higher accuracy of the estimated model. To avoid to move too far away from the original \UcbExplore, we decided to define the confidence intervals as if bucketing was used. We thus consider $\underline{N}_k(s,a) = N_k(s,a) / H$ for the construction of the confidence intervals.
For planning, we use the optimistic backward induction procedure as in~\cite{azar2017minimax}. We thus leverage empirical Bernstein inequalities ---which are much tighter--- rather than Hoeffding inequalities as suggested in~\cite{lim2012autonomous}.
In particular, we further approximate the bonus suggested in~\cite[][Alg.\,4]{azar2017minimax} as

\[
        b_h(s,a) = \sqrt{\frac{ Var_{s' \sim \hat{p}_k(\cdot|s,a)}[V_{k,h+1}(s')]}{\underline{N}_k(s,a) \lor 1}} + \frac{(H-h) }{\underline{N}_k(s,a) \lor 1}.
\]

For \ALGO, we follow the same approach of removing constants and logarithmic terms. We thus use the definition of $\phi$ as in Thm.\,\ref{thm:sample.comlexity} with $\alpha=1$ and without log-terms.
For planning, we use the procedure described in App.\,\ref{app_full_proof} with $b_k(s,a,s') = \sqrt{\frac{\wh p_k(s'|s,a)(1-\wh p_k(s'|s,a))}{N_k(s,a) \lor 1}} + \frac{1}{N_k(s,a) \lor 1}$. Finally, in the experiments we use a state-action dependent value $\wh{\Theta}(s,a, \cK_k) = \big(\sum_{s' \in \cK_k} \sqrt{ \wh{p}_k(s' \vert s,a)(1 - \wh{p}_k(s' \vert s,a))}\big)^2$ instead of taking the maximum over $(s,a)$.

\vspace{0.1in}

Even though we boosted the practical performance of \UcbExplore w.r.t.\,the original algorithm proposed in~\cite{lim2012autonomous} (e.g., the use of Bernstein), we believe it makes the comparison between \ALGO and \UcbExplore as fair as possible.

\subsection{Confusing Chain}\label{app:confusing.chain}

The \emph{confusing chain} environment referred to in Sect.\,\ref{sec:experiment} is constructed as follows. It is an MDP composed of an initial state $s_0$, a chain of length $C$ (states are denoted by $s_1, \ldots, s_C$) and a set of $K$ confusing states ($s_{C+1}, \ldots, s_{C+K}$). Two actions are available in each state. In state $s_0$, we have a forward action $a_0$ that moves to the chain with probability $p_c$ ($p(s_1|s_0,a_0) = p_c$ and $p(s_0|s_0,a_0) = 1-p_c$) and a confusing action that has uniform probability of reaching any confusing state ($p(s_i|s_0,a_1) = 1/K$ for any $i \in\{C+1, \ldots, C+K\}$). In the confusing states, all actions move deterministically to the end of the chain ($p(s_C|s_i,a) = 1$ for any $i \in\{C+1, \ldots, C+K\}$ and $a$).
In each state of the chain, there is a forward action $a_0$ that behaves as in $s_0$ ($p(s_{\min(C,i+1)}|s_i,a_0) = p_c$ and $p(s_i|s_i,a_0) = 1-p_c$, for any $i \in\{1, \ldots, C-1\}$) and a skip action $a_1$ that moves to $m$ states ahead with probability $p_{\mathrm{skip}}$ ($p(s_{\min(C,i+m)}|s_i,a_0) = p_{\mathrm{skip}}$ and $p(s_i|s_i,a_0) = 1-p_{\mathrm{skip}}$, for any $i \in\{1, \ldots, C-1\}$). Finally, $p(s_0|s_c, a) = 1$ for any action $a$. In our experiments, we set $m=4$, $p_{\mathrm{skip}} = 1/3$, $p_c=1$, $C=5$, $K=6$, $L=4.5$.

\newpage

\paragraph{Sample complexity.}

We provide in Tab.\,\ref{tab:new_results} the sample complexity of the algorithms for varying values of $\epsilon$. As mentioned in Sect.\,\ref{sec:experiment}, \ALGO outperforms \UcbExplore for any value of $\epsilon$, and increasingly so when $\epsilon$ decreases. Fig.\,\ref{fig:bbb} complements Fig.\,\ref{fig:3plots} for additional values of $\epsilon$.


\paragraph{Quality of goal-reaching policies.} We now investigate the quality of the policies recovered by \ALGO and \UcbExplore. In particular, we show that \ALGO is able to find the incrementally near-optimal shortest-path policies to any goal state, while \UcbExplore may only recover sub-optimal policies.  On the confusing chain domain, the intuition is that the set of confusing states makes $s_C$ reachable in just $2$ steps but the confusing states are not in the controllable set and thus the algorithms are not able to recover the shortest-path policy to $s_C$. On the other hand, state $s_C$ is controllable through two policies: 1) the policies $\pi_1$ that takes always the forward action $a_0$ reaches $s_C$ in $5$ steps; 2) the policy $\pi_2$ that takes the skip action $a_1$ in $s_1$ reaches $s_C$ in $4$ steps.
We observed empirically that \ALGO always recovers policy $\pi_1$ (i.e., the fastest policy) while \UcbExplore selects policy $\pi_2$ in several cases. This is highlighted in Tab.\,\ref{tab:expectedtime.confusing} where we report the expected hitting time of the policies recovered by the algorithms. This finding is not surprising since, as we explain in Sect.\,\ref{sect_analysis} and App.\,\ref{app_objectives}, \UcbExplore is designed to find policies reaching states in \textit{at most} $L$ steps on average, yet it is not able to recover incrementally near-optimal shortest-path policies, as opposed to \ALGO. 

\begin{table}[t]
        \centering
        \begin{tabular}{ccc}
                \hline
                $\epsilon$ & \ALGO  & \UcbExplore-Bernstein\\
                \hline
                $0.1$ & $ 374,263~(13,906)$ & $5,076,688~(92,643)$\\
                $0.2$ & $ 105,569~(4,645)$ & $636,580~(13,716)$ \\
                $0.4$ & $ 29,160~(829)$ & $108,894~(2,305)$\\
                $0.6$ & $ 15,349~(475)$ & $40,538~(805)$\\
                $0.8$ & $ 9,891~(244)$ & $21,270~(441)$\\
                \hline
        \end{tabular}
        \caption{Sample complexity of \ALGO and \UcbExplore-Bernstein, on the confusing chain domain. Values are averaged over $50$ runs and the $95\%$-confidence interval of the mean is reported in parenthesis.}
        \label{tab:new_results}
\end{table}

\begin{table}[t]
        \centering
        \begin{tabular}{ccccccc}
                \hline
                \multicolumn{7}{c}{\UcbExplore-Bernstein}\\
                \hline
                $\epsilon$ & \multicolumn{6}{c}{Expected hitting time $v_\pi(s_0 \to s_i)$}\\
                & $s_0$ & $s_1$ & $s_2$ & $s_3$ & $s_4$ & $s_5$\\
                \hline
                $0.1, 0.2$ & $0$ &  $1$ &  $2$  & $3$ &$4$ &$4$\\
                $0.4$ & $0$ &  $1$ &  $2$  & $3$ &$4$ &$4.94 ~(0.04)$\\
                $0.6$ & $0$ &  $1$ &  $2$  & $3.36~(0.11)$ &$4$ &$4.53 ~(0.07)$\\
                $0.8$ & $0$ &  $1$ &  $2$  & $3.38~(0.11)$ &$4.07~(0.07)$ &$4.53 ~(0.06)$\\
                \hline
        \end{tabular}

        \caption{Expected hitting time of state $s_i$ of the goal-oriented policy $\pi_{s_i}$ recovered by \UcbExplore-Bernstein, on the confusing chain domain. \ALGO recovers the optimal goal-oriented policy in all the runs and for all $\epsilon$. The advantage of \ALGO lies in its final policy consolidation step. Values are averaged over $50$ runs and the $95\%$-confidence interval of the mean is reported in parenthesis (it is omitted when equal to $0$). This shows that \UcbExplore recovers the optimal goal-oriented policy in every run only for $\epsilon$ equal to $0.1$ and $0.2$.}
        \label{tab:expectedtime.confusing}
\end{table}

\begin{figure}[t]
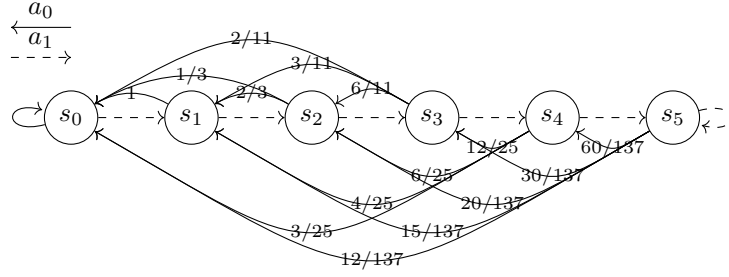

    \centering
    \tikz[scale=0.8]{
            \node[circle, draw] (S0) at (0,0) {$s_0$};
            \node[circle, draw] (S1) at (2,0) {$s_1$};
            \node[circle, draw] (S2) at (4,0) {$s_2$};
            \node[circle, draw] (S3) at (6,0) {$s_3$};
            \node[circle, draw] (S4) at (8,0) {$s_4$};
            \node[circle, draw] (S5) at (10,0) {$s_5$};
            \draw[<-] (-1,1.5) -- node[above]{$a_0$} (0, 1.5);
            \draw[->, dashed] (-1,1) -- node[above]{$a_1$} (0, 1);

            \path[->] (S0) edge[loop left] (S0);
            \draw[->, dashed] (S0) -- (S1);
            \draw[->, dashed] (S1) -- (S2);
            \draw[->, dashed] (S2) -- (S3);
            \draw[->, dashed] (S3) -- (S4);
            \draw[->, dashed] (S4) -- (S5);
            \path[->, dashed] (S5) edge[loop right] (S5);

            \path[->] (S1) edge[bend right] node[font=\scriptsize,] {$1$} (S0);

            \path[->] (S2) edge[bend right] node[font=\scriptsize,] {$2/3$} (S1);
            \path[->] (S2) edge[bend right] node[font=\scriptsize,] {$1/3$} (S0);

            \path[->] (S3) edge[bend right, looseness=1.4] node[font=\scriptsize,] {$6/11$} (S2) ;
            \path[->] (S3) edge[bend right, looseness=1.4] node[font=\scriptsize,midway] {$3/11$} (S1);
            \path[->] (S3) edge[bend right, looseness=1.4] node[font=\scriptsize,midway] {$2/11$} (S0);

            \path[->] (S4) edge[bend left, looseness=1.6] node[font=\scriptsize,midway] {$12/25$} (S3);
            \path[->] (S4) edge[bend left, looseness=1.6] node[font=\scriptsize,midway] {$6/25$} (S2);
            \path[->] (S4) edge[bend left, looseness=1.6] node[font=\scriptsize,midway] {$4/25$} (S1);
            \path[->] (S4) edge[bend left, looseness=1.6] node[font=\scriptsize,midway] {$3/25$} (S0);

            \path[->] (S5) edge[bend left, looseness=1.6] node[font=\scriptsize] {$60/137$} (S4) ;
            \path[->] (S5) edge[bend left, looseness=1.6] node[font=\scriptsize,midway] {$30/137$} (S3);
            \path[->] (S5) edge[bend left, looseness=1.6] node[font=\scriptsize,midway] {$20/137$} (S2);
            \path[->] (S5) edge[bend left, looseness=1.6] node[font=\scriptsize,midway] {$15/137$} (S1);
            \path[->] (S5) edge[bend left, looseness=1.6] node[font=\scriptsize,midway] {$12/137$} (S0);
    }
    \caption{Combination lock domain with $S=6$ states. Expected hitting times from the initial state $s_3$ are $v_{\pi} (s_3 \to s) = (2.18, 1.91, 1.64, 0,   1,   2)$. Consider $L=3$, the set of incrementally $L$-controllable states is $\mathcal{S}_L^{\rightarrow} = \{s_2, s_3, s_4, s_5\}$. The goal-oriented policy to reach $s_4$ and $s_5$ takes always the right action $a_1$, while the policy for $s_2$ always selects the left action $a_0$.}
    \label{fig:comb}
\end{figure}

\subsection{Combination Lock}
We consider the combination lock problem introduced in~\citep{azar2012dynamic}. The domain is a stochastic chain with $S=6$ states and $A=2$ actions. In each state $s_k$, action \emph{right} ($a_1$) is deterministic and leads to state $s_{k+1}$, while action \emph{left} ($a_0$) moves to a state $s_{k-l}$ with probability proportional to $1/(k-l)$ (i.e., inversely proportional to the distance of the states).
Formally, we have that
\begin{align*}
        n(x_k, x_l) = \begin{cases}
                \frac{1}{k-l} &\text{if } l < k\\
                0 & \text{otherwise}
        \end{cases}
        \qquad \text{and} \qquad
        p(x_l | x_k, a_0) = \frac{n(x_k, x_l)}{\sum_{s} n(x_k, s)}.
\end{align*}
We set the initial state to be at $2/3$ of the chain, i.e., $ \Big \lfloor 2 N / 3 \Big \rfloor$.
The actions in the end states are absorbing, i.e., $p(s_0|s_0, a_0) = 1$ and $p(s_{N-1}|s_{N-1}, a_1) = 1$, while the remaining actions behave normally. See Fig.~\ref{fig:comb} for an illustration of the domain.

\paragraph{Sample complexity.} We evaluate the two algorithms \ALGO and \UcbExplore on the combination lock domain, for $\epsilon=0.2$ and $L=2.7$. We further boost the empirical performance of \UcbExplore by using $N$ instead of $\underline{N}$ for the construction of the confidence intervals (i.e., we do not account for the data bucketing in \cite{lim2012autonomous}, see App.\,\ref{app_alg_configurations}). To preserve the robustness of the algorithm, we use $\log(|\mathcal{K}_k|^2)/ (\epsilon')^3$ episodes for \UcbExplore's policy evaluation phase (indeed we noticed that the removal of the logarithmic term here sometimes leads \UcbExplore to miss some states in~$\mathcal{S}_L^{\rightarrow}$ in this domain). For the same reason, in \ALGO we use the value $\wh\Theta(\mathcal{K}_k) = \max_{s,a} \wh\Theta(s,a,\mathcal{K}_k)$ prescribed by the theoretical algorithm instead of the state-action dependent values used in the previous experiment. We average the experiments over $20$ runs and obtain a sample complexity of $30,117$ ($2,087$) for \ALGO and $90,232$ ($2,592$) for \UcbExplore. Fig.\,\ref{fig:comb02} reports the proportion of incrementally $L$-controllable states identified by the algorithms as a function of time. We notice that once again \ALGO clearly outperforms \UcbExplore.

\begin{figure}[t]
    \centering
    \includegraphics[width=.4\textwidth]{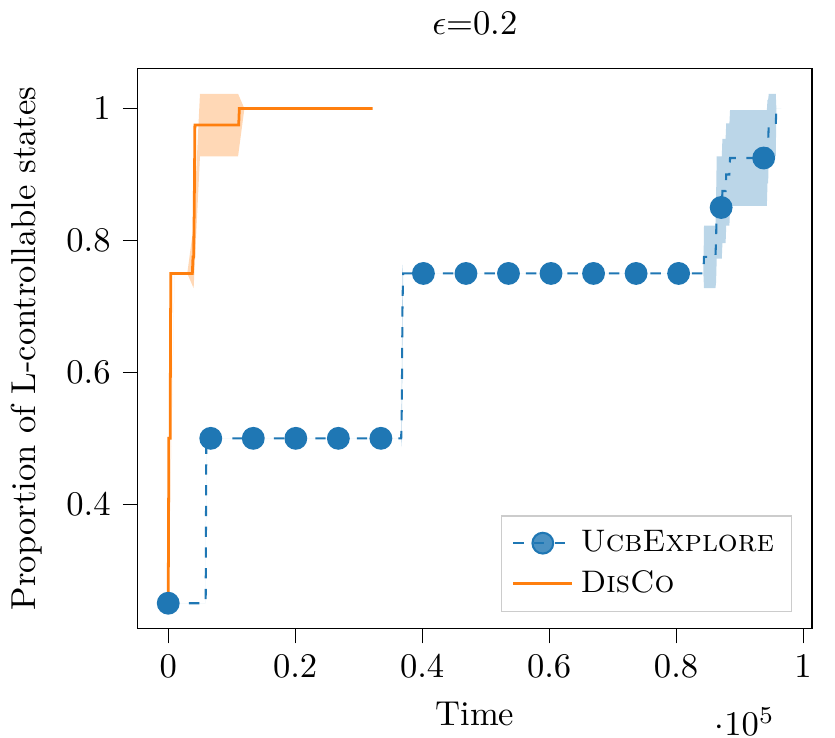}
    \caption{Proportion of the incrementally $L$-controllable states identified by \ALGO and \UcbExplore in the combination lock domain for $L = 2.7$ and $\epsilon=0.2$. Values are averaged over $20$ runs.}
    \label{fig:comb02}
\end{figure}

\begin{figure}[t]
        \centering
        \includegraphics[width=.32\textwidth]{skipdomain_proportion_01.pdf}
        \includegraphics[width=.32\textwidth]{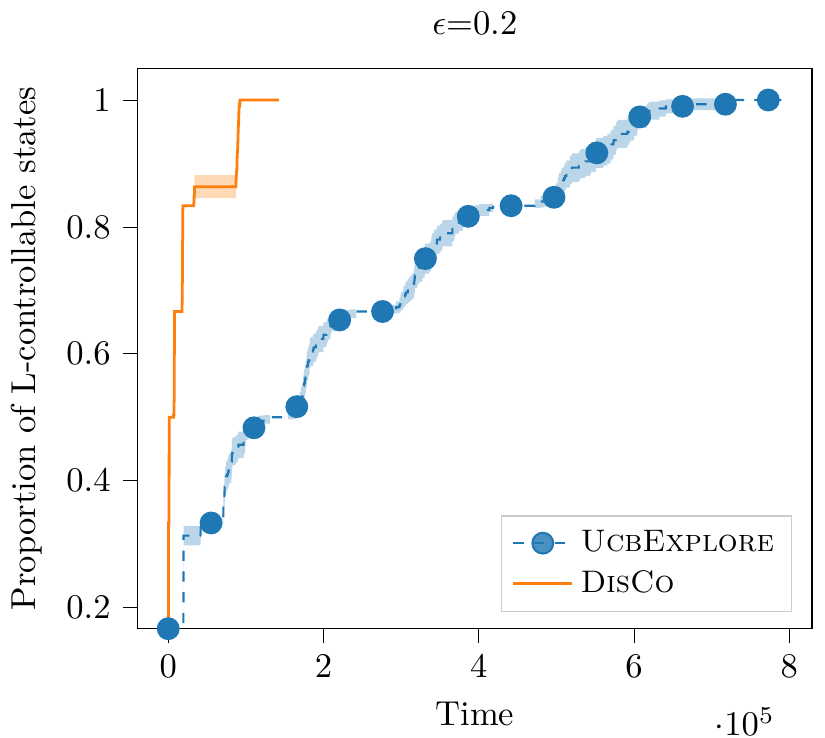}
        \includegraphics[width=.32\textwidth]{skipdomain_proportion_04.pdf}
        \includegraphics[width=.32\textwidth]{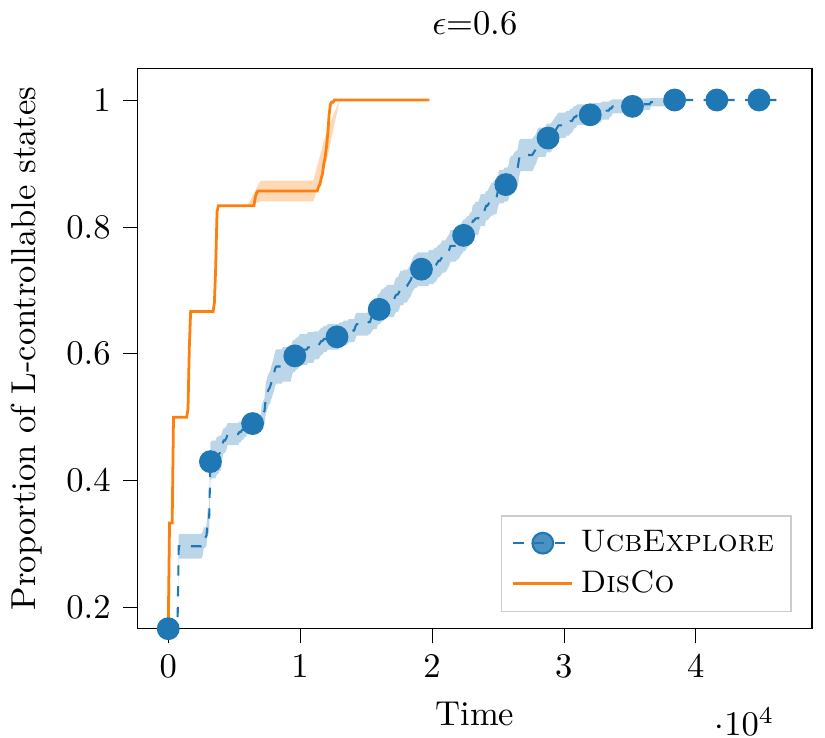}
        \includegraphics[width=.32\textwidth]{skipdomain_proportion_08.pdf}
        \caption{Proportion of the incrementally $L$-controllable states identified by \ALGO and \UcbExplore on the confusing chain domain for $L=4.5$ and $\epsilon \in \{0.1, 0.2, 0.4, 0.6, 0.8\}$. Values are averaged over $50$ runs. \UcbExplore uses Bernstein confidence intervals for planning.}
        \label{fig:bbb}
\end{figure}

\end{document}